\documentclass[journal]{IEEEtran}
\usepackage{amsmath,graphicx}
\usepackage{color}
\usepackage{lineno,hyperref}
\usepackage{amssymb}
\usepackage{multirow}
\usepackage{booktabs}
\usepackage{amsthm}
\usepackage{slashbox}
\usepackage{microtype}
\usepackage{threeparttable}
\usepackage{makecell}
\usepackage[ruled,linesnumbered]{algorithm2e}
\DeclareMathOperator*{\argmax}{argmax}
\DeclareMathOperator*{\argmin}{argmin}

\newtheorem{definition}{Definition}

\newtheorem{theorem}{Theorem}
\newtheorem{lemma}{Lemma}

\ifCLASSINFOpdf
\else
\fi

\newcommand{\change}[2]{}
\newcommand{\lchange}[2]{}

\newcommand{\changed}[3]{#3}
\newcommand{\lchanged}[3]{#3}

\begin{document}

\title{Robust Visual Tracking Revisited: From Correlation Filter to Template Matching}

\author{Fanghui Liu, Chen Gong, Xiaolin Huang, Tao Zhou,
Jie Yang, and Dacheng Tao

\thanks{
This work was supported in part by the National Natural Science Foundation of China under Grant 61572315, Grant 6151101179, Grant 61602246, Grant 61603248, in part by 973 Plan of China under Grant 2015CB856004, in part by the Natural Science Foundation of Jiangsu Province under Grant BK20171430, in part by the ``Six Talent Peak" Project of Jiangsu Province of China under Grant DZXX-027, in part by the China Postdoctoral Science Foundation under Grant No. 2016M601597, and in part by Australian Research Council Projects under Grant FL-170100117, Grant DP-180103424, Grant DP-140102164, and Grant LP-150100671. (\emph{Corresponding author: Jie Yang.})}
\thanks{
 F. Liu, X. Huang, T. Zhou and J. Yang are with Institute of Image Processing and Pattern Recognition, Shanghai Jiao Tong University, Shanghai 200240, China (e-mail: lfhsgre@sjtu.edu.cn; xiaolinhuang@sjtu.edu.cn; zhou.tao@sjtu.edu.cn; jieyang@sjtu.edu.cn).
 }
  \thanks{
 C. Gong is with the School of Computer Science and Engineering, Nanjing
University of Science and Technology, Nanjing 210094, China (e-mail: chen.gong@njust.edu.cn).}
\thanks{D. Tao is with the UBTECH Sydney Artificial Intelligence Centre and the School of Information Technologies, the Faculty of Engineering and Information Technologies, the University of Sydney, 6 Cleveland St, Darlington, NSW 2008, Australia (e-mail: dacheng.tao@sydney.edu.au).}
\thanks{\copyright 20XX IEEE. Personal use of this material is permitted. Permission from IEEE must be obtained for all other uses, in any current or future media, including reprinting/republishing this material for advertising or promotional purposes, creating new collective works, for resale or redistribution to servers or lists, or reuse of any copyrighted component of this work in other works.}
}
\markboth{}%
{Shell \MakeLowercase{\textit{et al.}}: Bare Demo of IEEEtran.cls for Journals}

\maketitle

\begin{abstract}
In this paper, we propose a novel matching based tracker by investigating the relationship between template matching and the recent popular correlation filter based trackers (CFTs).
Compared to the correlation operation in CFTs, a sophisticated similarity metric termed ``mutual buddies similarity" (MBS) is proposed to exploit the relationship of multiple reciprocal nearest neighbors for target matching.
By doing so, our tracker obtains powerful discriminative ability on distinguishing target and background as demonstrated by both empirical and theoretical analyses.
Besides, instead of utilizing single template with the improper updating scheme in CFTs, we design a novel online template updating strategy named ``memory filtering" (MF), which aims to select a certain amount of representative and reliable tracking results in history to construct the current stable and expressive template set.
This scheme is beneficial for the proposed tracker to comprehensively ``understand'' the target appearance variations, ``recall" some stable results. 
Both qualitative and quantitative evaluations on two benchmarks suggest that the proposed tracking method performs favorably against some recently developed CFTs and other competitive trackers.
\end{abstract}

\begin{IEEEkeywords}
visual tracking, template matching, mutual buddies similarity, memory filtering
\end{IEEEkeywords}

\IEEEpeerreviewmaketitle

\section{Introduction}
\label{sec:intro}
\IEEEPARstart{V}{isual} tracking is the problem of continuously localizing a pre-specified object in a video sequence.
Although much effort \cite{Choi2017CVPR_mini,Zhangle2017CVPR_mini,liu2017NMC_mini,Lan2014Multi_mini} has been
made, it still remains a challenging task to find a lasting solution for object tracking due to the intrinsic factors (\emph{e.g.}, shape deformation and rotation in-plane or out-of-plane) and extrinsic factors (\emph{e.g.}, partial occlusions and background clutter).
\begin{figure}
\begin{center}
\includegraphics[width=0.46\textwidth]{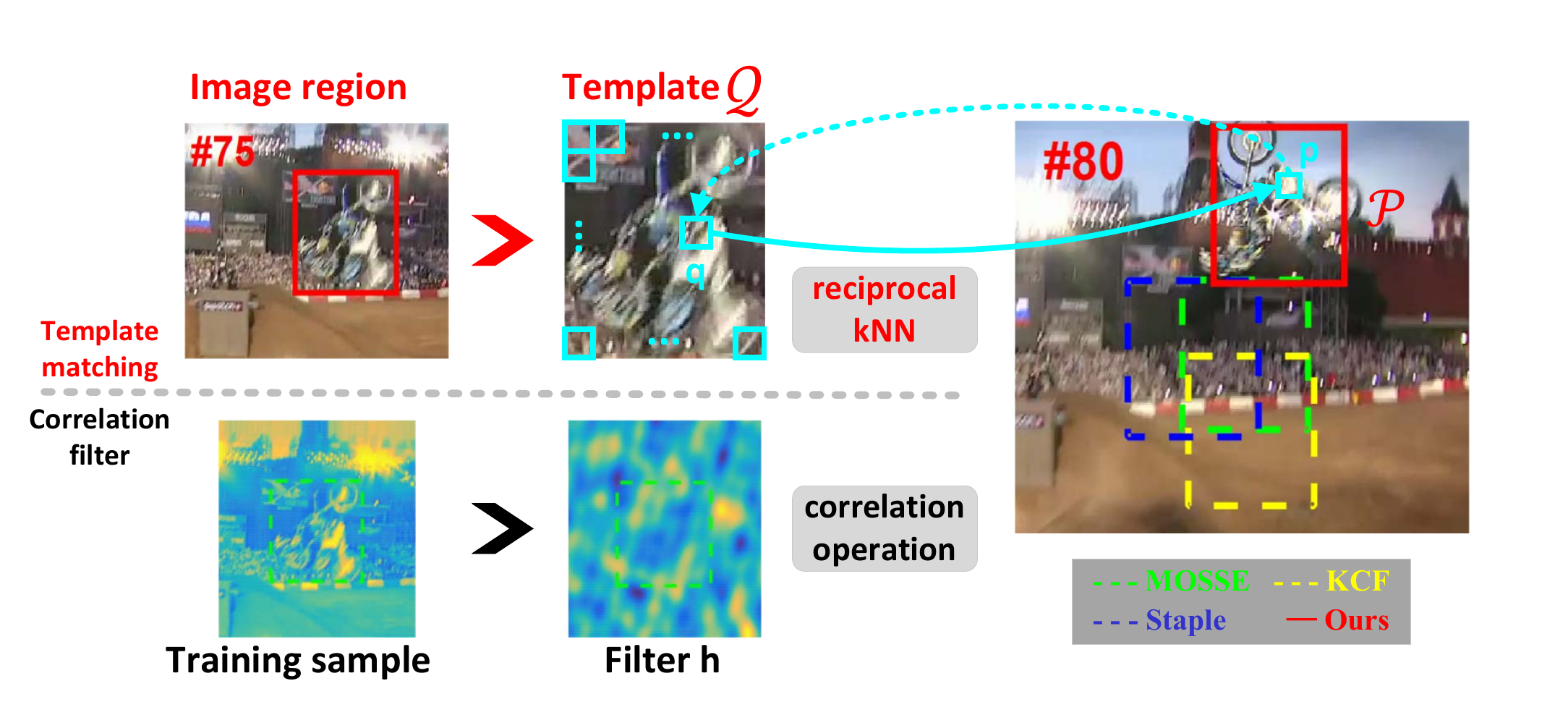}
\caption{\footnotesize Illustration of the relationship between our template matching based tracker (top panel) and CFTs (down panel). Our method considers a reciprocal $k$-NN scheme for similarity computation, and thus performs robustly to drastic appearance variations when compared to representative correlation filter based trackers including MOSSE \cite{Bolme2010_mini}, KCF \cite{henriques2015high_mini}, and Staple \cite{Bertinetto2016CVPR_mini}.
}
\label{vs}
\end{center}
\end{figure}

Recently, correlation filter based trackers (CFTs) have made significant achievements with high computational efficiency.
The earliest work was done by Bolme $et~al.$ \cite{Bolme2010_mini}, in which the filter $\mathbf{h}$ is learned by minimizing the total squared error between the actual output and the desired correlation output $\mathcal{Y}=\big\{ \mathbf{y}_i\big\}^{n}_{i=1}$ on a set of sample patches $\mathcal{X}=\big\{ \mathbf{x}_i\big\}^{n}_{i=1}$.
The target location can then be predicted by finding the maximum of the actual correlation response map $\mathbf{y}'$, that is computed as:
\begin{eqnarray*}\label{response}
  \mathbf{y}' = \mathcal{F}^{-1}(\hat{\mathbf{x}}'\odot \mathbf{\hat{h}}^*)\,,
\end{eqnarray*}
where $\mathcal{F}^{-1}$ is the inverse Fourier transform operation, and $\hat{\mathbf{x}}'$, $\mathbf{\hat{h}}$ are the Fourier transform of a new patch $\mathbf{x}'$ and the filter $\mathbf{h}$, respectively.
The symbol $*$ means the complex conjugate, and $\odot$ denotes element-wise multiplication.
In signal processing \changed{M2.4}{\link{R2.4}}{notation}, the learned filter $\mathbf{h}$ is also called a template, and accordingly, the correlation operation can be regarded as a similarity measure.
As a result, such CFTs share the similar framework with the conventional template matching based methods, as both of them aim to find the most similar region to \changed{M2.4}{\link{R2.4}}{the template} via a computational similarity metric (\emph{i.e.}, the correlation operation in CFTs or the reciprocal $k$-NN scheme in this paper) as shown in Fig.~\ref{vs}.

Although much progress has been made in CFTs, they often do not achieve satisfactory performance in some complex situations such as nonrigid object deformations and partial occlusions.
There are three reasons as follows.
First, by the correlation operation, all pixels within the candidate region $\mathbf{x}'$ and the template $\mathbf{h}$ are considered to measure their similarity.
In fact, a region may contain a considerable amount of redundant pixels that are irrelevant to its semantic meaning, and these pixels should not be taken into account for the similarity computation.
Second, the learned filter, as a single and global patch, often poorly approximates the object that undergoes nonlinear deformation and significant occlusions.
Consequently, it easily leads to model corruption and eventual failure.
Third, most CFTs usually update their models at each frame without considering whether the tracking result is accurate or not.

\lchanged{M1.5.1}{\link{R1.5},\link{R2.4}}{To sum up}, the tracking performance of CFTs is limited due to the direct correlation operation, and the single template with improper updating scheme.
Accordingly, we propose a novel template matching based tracker termed TM$^3$ ({\bf T}emplate {\bf M}atching via {\bf M}utual buddies similarity and {\bf M}emory filtering) tracker.
In TM$^3$, the similarity based reciprocal $k$ nearest neighbor is exploited to conduct target matching, and the scheme of memory filtering can select ``representative" and ``reliable" results to learn different types of templates.
By doing so, our tracker performs robustly to undesirable appearance variations.

\subsection{Related Work}
\label{sec:RW}
Visual tracking has been intensively studied and numerous trackers have been reviewed in the surveys such as \cite{Wu2015_mini,LiPRO2016_mini}.
Existing tracking algorithms can be roughly grouped into two categories: generative methods and discriminative methods.
Generative models aim to find the most similar region among the sampled candidate regions to the given target.
Such generative trackers are usually built on sparse representation \cite{Zhong2014_mini,Jia2016Visual_mini} or subspace learning \cite{Ross2008a_mini}.
The rationale of sparse representation based trackers is that the target can be represented by the atoms in \lchanged{M1.5.2}{\link{R1.5}}{an over-complete dictionary} with a sparse coefficient vector.
Subspace analysis utilizes PCA subspace \cite{Ross2008a_mini}, Riemannian manifold on a tangent space \cite{Li2008Robust_mini}, and other linear/nonlinear subspaces to model the relationship of object appearances.

In contrast, discriminative methods formulate object tracking as a classification problem, in which a classifier is trained to distinguish the foreground (\emph{i.e.}, the target) from the background.
Structured SVM \cite{Ning2016CVPR_mini} and \lchanged{M2.4}{\link{R2.4}}{the} correlation filter \cite{henriques2015high_mini,Bertinetto2016CVPR_mini} are representative tools for designing a discriminative tracker.
Structured SVM treats object tracking as a structured output prediction problem that admits a consistent target representation for both learning and detection.
CFTs train a filter, which encodes the target appearance, to yield strong response to a region that is similar to the target while suppressing responses to distractors.
Since our TM$^3$ method is based on the template matching mechanism, we briefly review several representative matching based trackers as follows.


Matching based trackers can be cast into two categories: template matching based methods and keypoint matching based approaches.
A template matching based tracker directly compares the image patches sampled from the current frame with the known template.
The primitive tracking method \cite{Sebastian2007Tracking_mini} \changed{M2.4}{\link{R2.4}}{employs} normalized cross-correlation (NCC) \cite{BriechleNCC_mini}.
The advantage of NCC lies in its simplicity for implementation and thus it is used in some recent trackers such as the TLD tracker \cite{kalal2012tracking_mini}.
Another representative template matching based tracker is proposed by Shaul $et~al.$ \cite{Oron2014klt_mini}, which extends \changed{M2.4}{\link{R2.4}}{the} Lucas-Kanade Tracking algorithm, and combines template matching with pixel based object/background segregation to build a unified Bayesian framework.
Apart from template matching, keypoint matching trackers have also gained much popularity and achieved a great success.
For example, Nebehay $et~al.$ \cite{nebehay2014consensus_mini} exploit keypoint matching and optical flow to enhance the tracking performance.
Hong $et~al.$ \cite{hong2015multi_mini} incorporate a RANSAC-based geometric matching for long-term tracking.
In \cite{Nebehay2015Clustering_mini}, by establishing correspondences on several deformable parts (served as key points) in an object, \changed{M2.4}{\link{R2.4}}{a} dissimilarity measure is proposed to evaluate their geometric compatibility for the clustering of correspondences, so the inlier correspondences can be separated from outliers.


\subsection{Our Approach and Contributions}
Based on the above discussion, we develop a similarity metric called ``Mutual Buddies Similarity" (MBS) to evaluate the similarity between two image regions\footnote{In our paper, an image region can be a template, or a candidate region such as a target region and a target proposal.}, based on the Best Buddies Similarity (BBS) \cite{Dekel_BBS_mini} that is originally designed for general image matching problem.
Herein, every image region is split into a set of \changed{M1.5.4}{\link{R1.5}}{non-overlapped} small image patches.
As shown in Fig.~\ref{vs}, we only consider the patches within the reciprocal nearest neighbor relationship, that is, one patch $p$ in a candidate region $\mathcal{P}$ is the nearest neighbor of the other one $q$ in the template $\mathcal{Q}$, and vice versa.
Thereby, the similarity computation relies on a subset of these ``reliable" pairs, and thus is robust to significant outliers and appearance variations.
Further, to improve the discriminative ability of this metric for visual tracking,
 the scheme of multiple reciprocal nearest neighbors is incorporated into the proposed MBS.
As shown in Fig.~\ref{knnfig}, for a certain patch $q$ in the target template $\mathcal{Q}$, only considering the $1$-reciprocal nearest neighbor of $q$ is definitely inadequate for a tracker to distinguish two similar candidate regions $\mathcal{P}_1$ and $\mathcal{P}_2$.
By exploiting these different 2nd, 3rd and 4th nearest neighbors, MBS can distinguish these candidate regions when they are extremely similar.
Moreover, such discriminative ability inherited by MBS is also theoretically demonstrated in our paper.
\begin{figure}
\begin{center}
\includegraphics[width=0.48\textwidth]{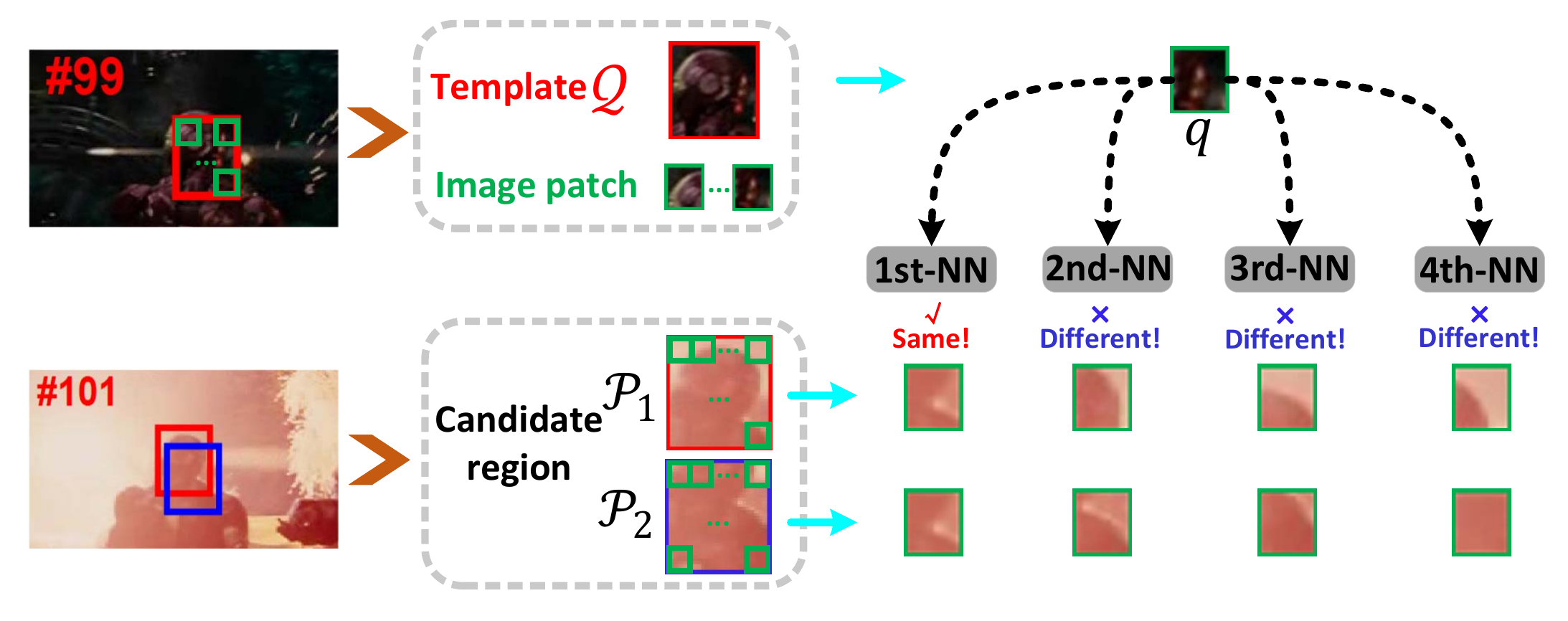}
\caption{\footnotesize Illustration of the nearest neighbors relationship of the given image patch $q$.
In our method, a template $\mathcal{Q}$, and two candidate regions $\mathcal{P}_1$ and $\mathcal{P}_2$  are split into a set of small image patches.
Specifically, we pick up a small patch $q$ in $\mathcal{Q}$ and then investigate its nearest neighbors (\emph{i.e.}, small image patches) in $\mathcal{P}_1$ and $\mathcal{P}_2$, respectively.
We see that two small patches selected as the 1st nearest neighbors of $q$ in $\mathcal{P}_1$ and $\mathcal{P}_2$ are the same; while its corresponding 2nd, 3rd and 4th nearest neighbors in these two regions are very dissimilar.
}
\label{knnfig}
\end{center}
\end{figure}

For template updating, two types of templates Tmpl$_r$ and Tmpl$_e$ are \changed{M1.5.5}{\link{R1.5}}{designed} in a comprehensive updating scheme.
The template Tmpl$_r$ is established by carefully selecting both ``representative" and ``reliable" tracking results \lchanged{M1.5.6}{\link{R1.5}}{during a long period}.
Herein, ``representative" denotes that a template should well represent the target under various appearances during the previous video frames.
The terminology ``reliable" implies that the selected template should be accurate and stable.
To this end, a memory filtering strategy is proposed in our TM$^3$ tracker to carefully pick up the representative and reliable tracking results during past frames, so that the accurate template Tmpl$_r$ can be properly constructed.
By doing so,  \lchanged{M2.4}{\link{R2.4}}{the} memory filtering strategy is able to ``recall" some stable results and ``forget" some results under abnormal situations.
Different from Tmpl$_r$, the template $\text{Tmpl}_e$ is frequently updated from the latest frames to timely adapt to the target's appearance changes within a short period.
Hence, compared to using the single template in CFTs, the combination of Tmpl$_e$ and Tmpl$_r$ is beneficial for our tracker to capture the target appearance variations in different periods.

Extensive evaluations on \lchanged{M2.4}{\link{R2.4}}{the} Object Tracking Benchmark (OTB) \cite{Wu2015_mini} (including OTB-50 and OTB-100) and Princeton Tracking Benchmark (PTB) \cite{song2013tracking_mini} suggest that in most cases our method significantly improves the performance of template matching based trackers and also performs favorably against the state-of-the-art trackers.


\label{sec:tracker}
\begin{figure*}
\begin{center}
\includegraphics[width=0.9\textwidth]{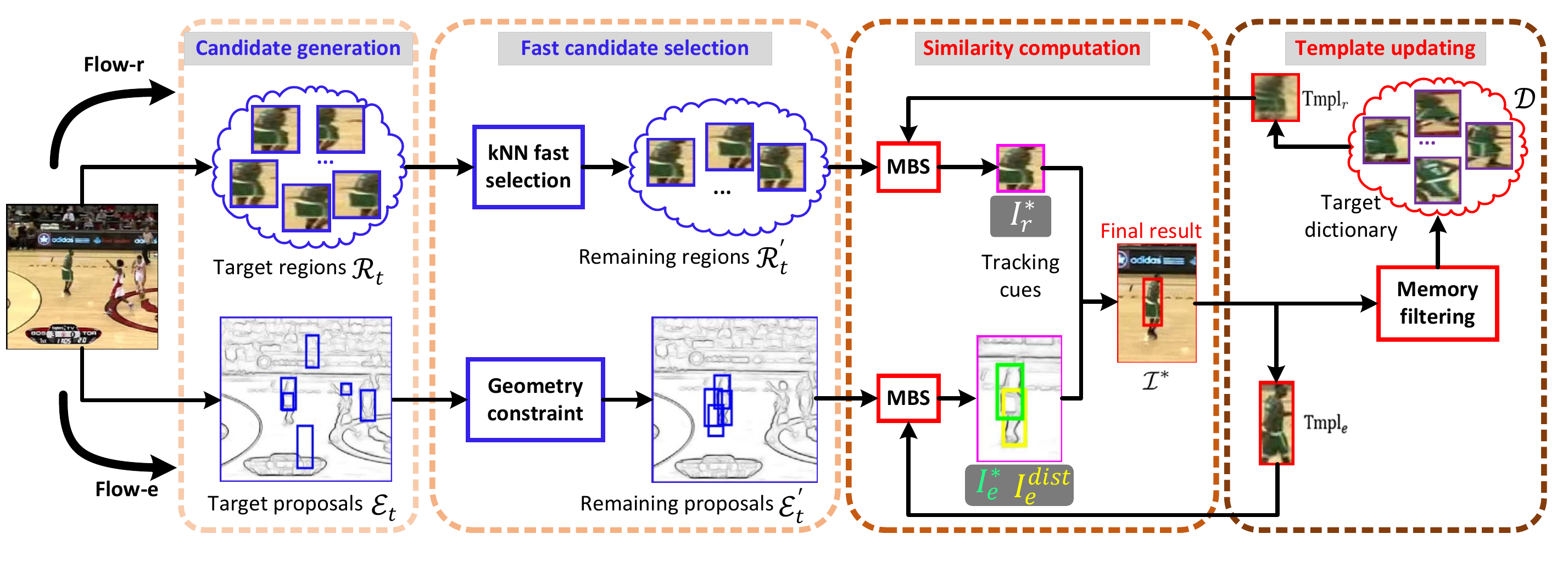}
\caption{\footnotesize The diagram of the proposed TM$^3$ tracker contains four main steps, namely \emph{candidate generation}, \emph{fast candidate selection}, \emph{similarity computation}, and \emph{template updating}.
Numerous potential candidate regions $\mathcal{R}_t\bigcup\mathcal{E}_t$ are produced in \emph{candidate generation} step, and then processed by the \emph{fast candidate selection} step to form the refined $\mathcal{R}'_t\bigcup\mathcal{E}'_t$.
In \emph{similarity computation} step, MBS is used to evaluate the similarities between these refined potential regions and the templates Tmpl$_r$, Tmpl$_e$, respectively.
Such step results in three tracking cues including $\mathcal{I}_r^*=\argmax \text{MBS}(\mathcal{R}'_t,\text{Tmpl}_r)$, $\mathcal{I}_e^*=\argmax \text{MBS}(\mathcal{E}'_t,\text{Tmpl}_e)$ with green box, and $\mathcal{I}_e^{dist}=\argmin dist(\mathcal{E}'_t,\mathcal{I}_r^*)$ with yellow box (the definition of distance metric $dist$ can be found in Section \ref{sec:screen}). The final tracking result $\mathcal{I}^*$ is jointly decided by above three cues. Finally, two types of templates $\text{Tmpl}_r$ and $\text{Tmpl}_e$ are updated via the memory filtering strategy. Above process iterates until all the frames of a video sequence have been processed.
}
\label{framework}
\end{center}
\end{figure*}

\section{The Proposed TM$^3$ Tracker}
 \label{sec:TM}
Our TM$^3$ tracker contains two main flows, \emph{i.e.} \texttt{Flow-r} ( \texttt{r} denotes random sampling) and  \texttt{Flow-e} (\texttt{e} is short for EdgeBox).
Each flow will undergo four steps, namely \emph{candidate generation}, \emph{fast candidate selection}, \emph{similarity computation}, and \emph{template updating} as shown in \lchanged{M1.5.10}{\link{R1.5}}{Fig.~\ref{framework}}.

In \texttt{Flow-r} process, at the $t$th frame, the $N_r$ target regions $\mathcal{R}_t=\{\mathcal{I}(\mathbf{c}_r^i,s_r^i)\}_{i=1}^{N_r}$ are randomly sampled from a Gaussian distribution, where $\mathcal{I}(\mathbf{c}_r^i,s_r^i)$\footnote{Note that the time index $t$ is omitted and we denote $\mathcal{I}(\mathbf{c}_r^i,s_r^i)$ as $\mathcal{I}_r^i$ for simplicity.} is the $i$th image region with the center coordinate $\mathbf{c}_r^i=[c_r^{ix},c_r^{iy}]$ and scale factor $s_r^i$.
To efficiently reduce the \lchanged{M2.4}{\link{R2.4}}{computational} complexity, a fast $k$-NN selection algorithm \cite{zhou2016graph_mini} is used to form the refined candidate regions $\mathcal{R}'_t=\{\mathcal{I}(\mathbf{c}_r^i,s_r^i)\}_{i=1}^{N'_r}$ that are composed of $N'_{r}$ nearest neighbors of the tracking result at the $(t-1)$th frame from $\mathcal{R}_t$.
After that, MBS evaluates the similarities between these target regions and the template Tmpl$_r$, and then outputs $\mathcal{I}(\mathbf{c}_r^*,s_r^*)$ with the highest similarity score.

In \texttt{Flow-e} process, at the $t$th frame, the EdgeBox approach \cite{zitnick2014edge_mini} is utilized to generate a set $\mathcal{E}_t=\{\mathcal{I}(\mathbf{c}_e^j,s_e^j)\}_{j=1}^{N_e}$ with $N_e$ target proposals, where $\mathcal{I}(\mathbf{c}_e^j,s_e^j)$ ($\mathcal{I}_e^j$ for simplicity) is the $j$th image region with the center coordinate $\mathbf{c}_e^j=[c_e^{jx},c_e^{jy}]$ and scale factor $s_e^j$.
Subsequently, \changed{M1.5.7}{\link{R1.5}, \link{R2.4}}{most non-target proposals} are filtered out by exploiting the geometry constraint.
After that, only a small amount of $N'_{e}$ potential proposals $\mathcal{E}'_t=\{\mathcal{I}(\mathbf{c}_e^j,s_e^j)\}_{j=1}^{N'_{e}}$ are evaluated by MBS to indicate \changed{M1.5.8}{\link{R1.5}}{how they are similar to} the template Tmpl$_e$.
This process generates two tracking cues $\mathcal{I}_e^*$ and $\mathcal{I}_e^{dist}$, which will be further described in Section \ref{sec:screen}.
Based on the three tracking cues generated by above \changed{M2.4}{\link{R2.4}}{the} two flows, the final tracking result $\mathcal{I}^*$ at the $t$th frame is obtained by fusing $\mathcal{I}_r^*$, $\mathcal{I}_e^*$ and $\mathcal{I}_e^{dist}$ with details introduced in Section \ref{sec:fuse}.

Note that in the entire tracking process, illustrated in Fig.~\ref{framework}, similarity computation and template updating are critical for our tracker to achieve satisfactory performance, so we will detail them in Sections \ref{sec:MBS} and  \ref{sec:al}, respectively.
\subsection{Mutual Buddies Similarity}
\label{sec:MBS}
In this section, we detail the MBS for computing the similarity between two image regions, which aims to improve the discriminative ability of our tracker.

\subsubsection{The definition of MBS}
\label{sec:MBSd}
 In our TM$^3$ tracker, every image region is split into a set of non-overlapped $3 \times 3$ image patches.
Without loss of generality, we denote a candidate region as a set $\mathcal{P}=\{ \mathbf{p}_i \}_{i=1}^N$, where $\mathbf{p}_i$ is a small patch and $N$ is the number of such small patches.
Likewise, a template (Tmpl$_r$ or Tmpl$_e$) is represented by the set of $\mathcal{Q}=\{ \mathbf{q}_j \}_{j=1}^M$ of size $M$.
The \changed{M2.4}{\link{R2.4}}{objective} of MBS is to reasonably evaluate the similarity between $\mathcal{P}$ and $\mathcal{Q}$, so that a faithful and discriminative similarity score can be assigned to the candidate region $\mathcal{P}$.
%

To design MBS, we begin with the definition of a similarity metric MBP between two patches $\{ \mathbf{p}_i \in \mathcal{P}, \mathbf{q}_j \in \mathcal{Q}\}$.
\changed{M1.5.9}{\link{R1.5}, \link{R2.4}}{Assuming that $\mathbf{q}_j$ is the $r$th nearest neighbor of $\mathbf{p}_i$ in the set of $\mathcal{Q}$ (denoted as $\mathbf{q}_j = \text{NN}_r(\mathbf{p}_i,\mathcal{Q})$), and meanwhile $\mathbf{p}_i$ is the $s$th nearest neighbor of $\mathbf{q}_j$ in $\mathcal{P}$ (denoted as $\mathbf{p}_i =\text{NN}_s(\mathbf{q}_j,\mathcal{P})$),} then the similarity \text{MBP} of two patches $\mathbf{p}_i$ and $\mathbf{q}_j$ is:
\begin{equation}\label{mBBPdef}
  \text{MBP}(\mathbf{p}_i,\!\mathbf{q}_j) =
  \mathrm{e}^{-\frac{rs}{\sigma_1}},\!~\text{if}~ \mathbf{q}_j = \text{NN}_r(\mathbf{p}_i,\mathcal{Q})  \wedge
 \mathbf{p}_i =\text{NN}_s(\mathbf{q}_j,\mathcal{P})\,,
\end{equation}
where  \changed{M2.4.1}{\link{R2.4}}{$\sigma_1$ is a tuning parameter.
In our experiment, we empirically set it to 0.5.}
Such similarity metric \text{MBP} evaluates the closeness level between two patches by the scheme of multiple reciprocal nearest neighbors. 
Therefore, MBS between $\mathcal{P}$ and $\mathcal{Q}$ is defined as\footnote{In the experimental setting, the set sizes of $\mathcal{P}$ and $\mathcal{Q}$ have been set to the same value.}:
\begin{equation}\label{MBS}
 \text{MBS}(\mathcal{P},\mathcal{Q}) = \frac{1}{\min\{M,N\}} \cdot \sum_{i=1}^N\sum_{j=1}^M \text{MBP} (\mathbf{p}_i,\mathbf{q}_j)\,,
\end{equation}
One can see that MBS is the statistical average of MBP.
Specifically, the similarity metric BBP used in BBS \cite{Dekel_BBS_mini} is defined as:
\begin{equation*}\label{BBPdef}
 {\text{BBP}}(\mathbf{p}_i,\mathbf{q}_j)\!=\!\! \left\{
\begin{array}{rcl}
\begin{split}
&\!\!1  ~~ \mathbf{q}_j \!=\! \text{NN}_1(\mathbf{p}_i,\!\mathcal{Q})  \wedge
 \mathbf{p}_i \!= \!\text{NN}_1(\mathbf{q}_j,\!\mathcal{P}); \\
&\!\!0     ~~\text{otherwise}. \\
\end{split}
\end{array} \right.
\end{equation*}
Herein, the metric ${\text{BBP}}$ can be viewed as a special case of the proposed similarity \text{MBP} in Eq.~\eqref{mBBPdef} when $r$ and $s$ are set to 1.

As aforementioned, BBS shows less discriminative ability than MBS for object tracking, and next we will \lchanged{M1.5.11}{\link{R1.5}}{ theoretically} explain the reason.
Note that the factor $ \frac{1}{\min\{M,N\}}$ defined in Eq.~\eqref{MBS} does not \lchanged{M1.5.12}{\link{R1.5}}{have influence on} the theoretical result, so we  investigate the relationship between BBS and MBS without any factor to verify the effectiveness of such multiple nearest neighbor scheme.
To this end, we introduce first-order statistic (mean value) and second-order statistic (variance) of \lchanged{M1.5.13}{\link{R1.5},\link{R2.4}}{two such} metrics to analyse their respective discriminative (or ``scatter") ability.
We begin with the following statistical definition:
\begin{definition}\label{definite}
Suppose two image patches $\mathbf{p}_i$ and $\mathbf{q}_j$ are randomly sampled from two given distributions $\mathrm{Pr}\{P\}$ and $\mathrm{Pr}\{Q\}$\footnote{Such general definition does not rely on a specific distribution of $\mathrm{Pr}\{P\}$ and $\mathrm{Pr}\{Q\}$. The details of $\mathbb{E}_{\text{BBS}}(\mathcal{P},\mathcal{Q})$ can be found in Eq.~(4) in \cite{Dekel_BBS_mini}.} corresponding to the sets $\mathcal{P}$ and $\mathcal{Q}$, respectively,
and $\mathbb{E}_{\text{MBS}}(\mathcal{P},\mathcal{Q})$ is the expectation of the similarity score between a pair of patches $\{\mathbf{p}_i, \mathbf{q}_j\}$ computed by \text{MBP} over all possible pairwise patches in $\mathcal{P}$ and $\mathcal{Q}$, then we have:
\begin{equation}\label{EBBP}
\begin{split}
 &\mathbb{E}_{\text{MBS}}(\mathcal{P},\mathcal{Q})\! =\!\!\! \int_{P}\!\int_{Q}\! {\text{MBP}}(\mathbf{p}_i,\mathbf{q}_j) \mathrm{Pr}\{P\}
  \mathrm{Pr}\{Q\} \mathrm{d}P \mathrm{d}Q\,,\\
  &\mathbb{E}_{\text{BBS}}(\mathcal{P},\mathcal{Q})\! =\!\!\! \int_{P}\!\int_{Q}\! {{\text{BBP}}}(\mathbf{p}_i,\mathbf{q}_j) \mathrm{Pr}\{P\}
  \mathrm{Pr}\{Q\} \mathrm{d}P \mathrm{d}Q\,.
  \end{split}
\end{equation}
\end{definition}
By this definition, the variance $\mathbb{V}_{\text{MBS}}(\mathcal{P},\mathcal{Q})$ and $\mathbb{V}_{\text{BBS}}(\mathcal{P},\mathcal{Q})$ can be easily computed.
Formally, we have the following three lemmas.
We begin with the simplification of $\mathbb{E}_{\text{MBS}}(\mathcal{P},\mathcal{Q})$ in Lemma \ref{proEsim}, and next compute $\mathbb{E}_{\text{MBS}^2}(\mathcal{P},\mathcal{Q})$ and $\mathbb{E}_{\text{MBS}}^2(\mathcal{P},\mathcal{Q})$ to obtain $\mathbb{V}_{\text{MBS}}(\mathcal{P},\mathcal{Q})$ in Lemma \ref{proDsim}.
Lastly, we seek for the relationship between $\mathbb{E}_{\text{MBS}^2}(\mathcal{P},\mathcal{Q})$ and $\mathbb{E}_{\text{MBS}}(\mathcal{P},\mathcal{Q})$ in Lemma \ref{Lemmare}.
\changed{M1.5.14}{\link{R1.5}}{The proofs of these three lemmas} are put into Appendix \ref{sec:app1}, \ref{sec:app2} and \ref{sec:app3}, respectively.
\begin{lemma}\label{proEsim}
Let $F_{P}(x)$, $F_{Q}(x)$ be the cumulative distribution functions (CDFs) of $\mathrm{Pr}\{P\}$ and $\mathrm{Pr}\{Q\}$, respectively, and
assuming that each patch is independent of the others, then the multivariate integral in Eq.~\eqref{EBBP} can be represented by Eq.~\eqref{Esim}, where $p^+=p+d(p,q)$,  $p^-=p-d(p,q)$, $q^+=q+d(p,q)$, and $q^-=q-d(p,q)$.
\end{lemma}

\begin{lemma}\label{proDsim}
Given $\mathbb{E}_{\text{MBS}}(\mathcal{P},\mathcal{Q})$ obtained in Lemma \ref{proEsim},
the variance $\mathbb{V}_{\text{MBS}}(\mathcal{P},\mathcal{Q})$ can be computed by Eq.~\eqref{propoD}.
\end{lemma}

\begin{lemma}\label{Lemmare}
The relationship between $\mathbb{E}_{\text{MBS}^2}(\mathcal{P},\mathcal{Q})$ and $\mathbb{E}_{\text{MBS}}(\mathcal{P},\mathcal{Q})$ satisfies:
\begin{equation}\label{relationEsim}
 \mathbb{E}_{\text{MBS}^2}(\mathcal{P},\mathcal{Q}) > \mathbb{E}_{\text{MBS}}(\mathcal{P},\mathcal{Q})\,.
\end{equation}
\end{lemma}

\newcounter{mytempeqncnt}
\begin{figure*}[!t]
\normalsize
\begin{equation}\label{Esim}
\begin{split}
&\mathbb{E}_{\text{MBS}}(\mathcal{P},\mathcal{Q}) = 1-\frac{1}{\sigma_1}MN \iint\limits_{-\infty}^{~~~~\infty} \big[F_{P}(q^+)-F_{P}(q^-)\big]
\big[F_{Q}(p^+)-F_{Q}(p^-)\big]
f_{P}(p)f_{Q}(q)
\mathrm{d}p\mathrm{d}q\\
&~~~~~+\frac{1}{2\sigma_1^2}M^2N^2\iint\limits_{-\infty}^{~~~~\infty}
\big[F_{P}(q^+)-F_{P}(q^-)\big]^2
\big[F_{Q}(p^+)-F_{Q}(p^-)\big]^2
f_{P}(p)f_{Q}(q)\mathrm{d}p\mathrm{d}q\,.
\end{split}
\end{equation}
\begin{equation}\label{propoD}
\begin{split}
&\mathbb{V}_{\text{MBS}}(\mathcal{P},\mathcal{Q}) =\frac{1}{\sigma_1^2}M^2N^2\iint\limits_{-\infty}^{~~~~\infty}
\big[F_{P}(q^+)-F_{P}(q^-)\big]^2\big[F_{Q}(p^+)-F_{Q}(p^-)\big]^2
f_{P}(p)f_{Q}(q)\mathrm{d}p\mathrm{d}q\\
&~~~~~~~
-\frac{1}{\sigma_1^2}M^2N^2 \Bigg\{\iint\limits_{-\infty}^{~~~~\infty} \big[F_{P}(q^+)-F_{P}(q^-)\big]
\big[F_{Q}(p^+)-F_{Q}(p^-)\big]f_{P}(p)f_{Q}(q)
\mathrm{d}p\mathrm{d}q\Bigg\}^2\,.
\end{split}
\end{equation}
\vspace*{-4pt}
\end{figure*}

By introducing above three auxiliary Lemmas, we formally present Theorem \ref{theor} as follows.
\begin{theorem}\label{theor}
The relationship between $\mathbb{V}_{\text{MBS}}(\mathcal{P},\mathcal{Q})$ and $\mathbb{V}_{\text{BBS}}(\mathcal{P},\mathcal{Q})$ satisfies:
\begin{equation}\label{fb}
\mathbb{V}_{\text{MBS}}(\mathcal{P},\mathcal{Q})> \mathbb{V}_{\text{BBS}}(\mathcal{P},\mathcal{Q}), ~if~ \mathbb{E}_{\text{MBS}}(\mathcal{P},\mathcal{Q})=
\!\mathbb{E}_{\text{BBS}}(\mathcal{P},\mathcal{Q}).
\end{equation}
\end{theorem}
\begin{proof}
We firstly obtain $\mathbb{V}_{\text{BBS}}(\mathcal{P},\mathcal{Q})$ and then prove that under the condition of $\mathbb{E}_{\text{MBS}}(\mathcal{P},\mathcal{Q})=
\mathbb{E}_{\text{BBS}}(\mathcal{P},\mathcal{Q})$, the variance of $\text{MBS}(\mathcal{P},\mathcal{Q})$ is larger than that of $\text{BBS}(\mathcal{P},\mathcal{Q})$.

Since $\big[{{\text{BBP}}}(\mathbf{p}_i,\mathbf{q}_j)\big]^2
={{\text{BBP}}}(\mathbf{p}_i,\mathbf{q}_j)$ is derived from Eq.~\eqref{BBPdef}, we have $\mathbb{E}_{{\text{BBS}}^2}(\mathcal{P},\mathcal{Q})=
\mathbb{E}_{\text{BBS}}(\mathcal{P},\mathcal{Q})$\footnote{Note that we denote $\mathbb{E}_{\text{BBS}}(\mathcal{P},\mathcal{Q})$ as $\mathbb{E}_{\text{BBS}}$ by dropping ``$(\mathcal{P},\mathcal{Q})$'' for simplicity in the following description.} by Definition~\ref{definite}.
\changed{M1.1}{\link{R1.1}}{Based on this, under the condition of $\mathbb{E}_{\text{MBS}}(\mathcal{P},\mathcal{Q})=
\mathbb{E}_{\text{BBS}}(\mathcal{P},\mathcal{Q})$, we have:
\begin{equation}\label{req}
\begin{split}
&\mathbb{V}_{\text{MBS}}\!-\!
\mathbb{V}_{\text{BBS}}
=\mathbb{E}_{\text{MBS}^2}-
[\mathbb{E}_{\text{MBS}}]^2
-\mathbb{E}_{\text{BBS}^2}
+[\mathbb{E}_{\text{BBS}}]^2\\
&~~~~~~~~~~~~~~~=\mathbb{E}_{\text{MBS}^2}-\mathbb{E}_{\text{BBS}^2} ~ (\text{Using}~ \mathbb{E}_{\text{MBS}}=
\mathbb{E}_{\text{BBS}}) \\
&~~~~~~~~~~~~~~~=\mathbb{E}_{\text{MBS}^2}-\mathbb{E}_{\text{BBS}} ~ (\text{Using}~\mathbb{E}_{{\text{BBS}}^2}=\mathbb{E}_{\text{BBS}})\\
&~~~~~~~~~~~~~~~=\mathbb{E}_{\text{MBS}^2}-\mathbb{E}_{\text{MBS}} >0\,.
\end{split}
\end{equation}
Finally, we obtain $\mathbb{V}_{\text{MBS}}(\mathcal{P},\mathcal{Q}) > \mathbb{V}_{\text{BBS}}(\mathcal{P},\mathcal{Q})$ as claimed, thereby completing the proof.
}
\end{proof}

Theorem~\ref{theor} theoretically demonstrates that under the condition of $\mathbb{E}_{\text{MBS}}(\mathcal{P},\mathcal{Q})=
\mathbb{E}_{\text{BBS}}(\mathcal{P},\mathcal{Q})$, the variance of $\text{MBS}(\mathcal{P},\mathcal{Q})$ is larger than that of $\text{BBS}(\mathcal{P},\mathcal{Q})$, which indicates that MBS is able to produce more disperse similarity scores over $\mathcal{P}$ and $\mathcal{Q}$ than BBS when they have the same mean value of similarity scores.
Therefore, MBS equipped with the scheme of multiple reciprocal nearest neighbors is more discriminative than BBS \lchanged{M1.5.16}{\link{R1.5}}{for distinguishing} numerous candidate regions when they are extremely similar, which can be illustrated in Fig.~\ref{rankingfig}.
It can be observed that the similarity score curve of BBS (see blue curve) within No.1$\sim$No.57, No.58$\sim$No.89, and No.90$\sim$No.100 candidate regions is almost flat, which \lchanged{M1.5.17}{\link{R1.5}}{means} that BBS cannot tell a difference on these ranges.
In contrast, the similarity scores decided by MBS (see red curve) on all candidate regions are totally different and thus are discriminative.
By such scheme, more image patches are involved in computing the similarity scores and they yield different responses to candidate regions.
Accordingly, MBS can distinguish numerous candidate regions when they are extremely similar, and thus the discriminative ability of our tracker can be effectively improved.

\begin{figure}
\begin{center}
\includegraphics[width=0.44\textwidth]{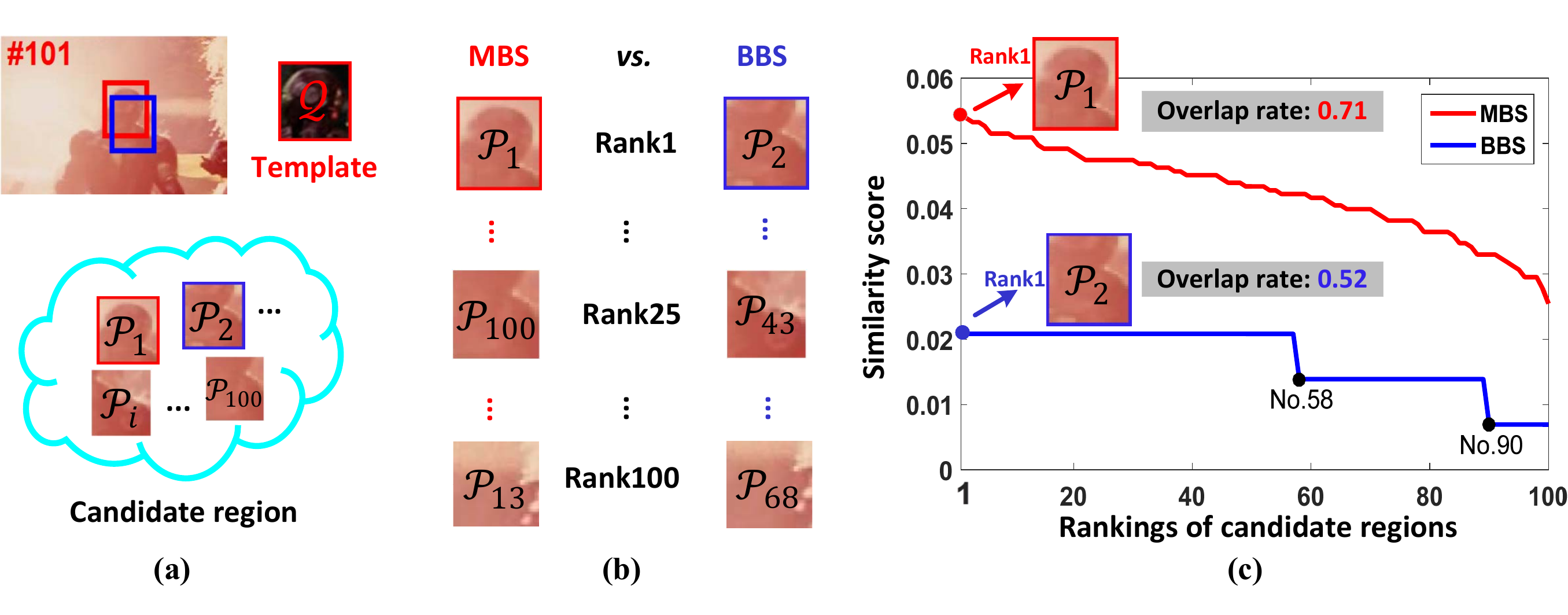}
\caption{\footnotesize Illustration of the superiority of MBS to BBS.
We aim to decide which of the 100 candidate regions ( from $\mathcal{P}_1$ to $\mathcal{P}_{100}$) in (a) mostly matches the target template $\mathcal{Q}$.
The overlap rate between the decided target and groundtruth region is particularly observed.
We see that the candidate region $\mathcal{P}_1$ with the highest similarity score is selected by MBS; while the inferior result $\mathcal{P}_2$ is picked up by BBS.
Consequently, $\mathcal{P}_1$ computed by the MBS achieves higher overlap rate (71\%) than $\mathcal{P}_2$ (52\%) that is obtained by the BBS.
Specifically, in (c), after ranking these candidates according to their similarity scores measured by MBS, we plot their rankings on the horizontal axis and the corresponding similarity scores computed by MBS (red curve) and BBS (blue curve) on the vertical axis.
}
\label{rankingfig}
\end{center}
\end{figure}

\subsection{Memory Filtering for Template Updating}
\label{sec:al}
Here we investigate the template updating scheme in our TM$^3$ tracker by designing two types of templates Tmpl$_r$ and Tmpl$_e$.
For Tmpl$_e$, the tracking result in the current frame is directly taken as the template Tmpl$_e$ if its similarity score is larger than a pre-defined threshold $0.5$.
Note that low threshold would incur in some unreliable results and thus degrade the tracking accuracy; a large one makes the template Tmpl$_e$ difficult to be frequently updated.
In other words, Tmpl$_e$ is frequently updated to capture the target appearance change in a short period without error accumulation.
In contrast, the template Tmpl$_r$ focuses on the tracking results in history, and it is updated via the memory filtering strategy to ``recall" some stable results and ``forget" some results under abnormal situations.

\subsubsection{Formulation of memory filtering}
Suppose the tracked target in every frame is characterized by a $d$-dimensional feature vector, then the tracking results in the latest $N_s$ frames can be arranged as a data matrix $\mathbf{X} \in \mathbb{R}^{N_s\times d}$ where each row represents a tracking result.
Similar to sparse dictionary selection \cite{Krause2010Submodular_mini}, memory filtering aims to find a compact subset from $N_s$ tracking results so that they can well represent the entire $N_s$ results.
To this end, we define a selection matrix $\mathbf{S} \in \mathbb{R}^{N_s\times N_s}$ of which $s_{ij}$ reflects the expressive power of the $i$th tracking result $\mathbf{x}_i$ \changed{M1.5.18}{\link{R1.5}}{on} the $j$th result $\mathbf{x}_j$, so the norm of the $\mathbf{S}$'s $i$th row suggests \changed{M1.5.18}{\link{R1.5}}{the qualification of} $\mathbf{x}_i$ is to represent the whole $N_s$ results.
The selection process is fulfilled by solving:
\begin{equation}\label{mainss}
  \mathop{\mathrm{min}}\limits_{\mathbf{S}}  \underbrace{\frac{1}{2}\| \mathbf{X}-\mathbf{X}\mathbf{S}\|^2_\mathrm{F} + \delta \mathrm{Tr}(\mathbf{S}^{\top}\mathbf{L}\mathbf{S})}_{\triangleq f(\mathbf{S})} +  \underbrace{\beta \sum_{i=1}^{N_s} \frac{1}{h_i+\varepsilon} \|\mathbf{S}\|_{1,2}}_{\triangleq g(\mathbf{S})},
\end{equation}
where \changed{M2.4}{\link{R2.4}}{$\varepsilon$ is a small positive constant to avoid being divided by zero}, and  $\|\mathbf{S}\|_{1,2}=\sum_{i=1}^{N_s}\|\mathbf{S}_{i,.}\|_2$ denotes the sum of $\ell_2$ norm of all $N_s$ rows.
\changed{M1.5.19}{\link{R1.5}}{The second term in Eq.~(\ref{mainss}) is the smoothness graph regularization with the weighting parameter $\delta$.
Within this term, the similar tracking results will have a similar probability to be selected.}
Herein, the Laplacian matrix is defined by $\mathbf{L} = \mathbf{D} - \mathbf{W}$, where $\mathbf{D}$ is a diagonal matrix with $\mathbf{D}_{ii}=\sum_j\mathbf{W}_{ij}$ and $\mathbf{W}$ is the weight matrix defined by the reciprocal nearest neighbors scheme, namely:
\begin{equation*}\label{knn}
 \mathbf{W}_{ij}\!=\mathrm{e}^{-\frac{rs}{\sigma_2}}~~ \text{if} ~~\mathbf{x}_i \!\!=\! \mathcal{N}_r(\mathbf{x}_j,\!\mathbf{X})  \wedge
 \mathbf{x}_j \!\!=\! \!\mathcal{N}_s(\mathbf{x}_i,\!\mathbf{X})\,,
\end{equation*}
where $\sigma_2=2$ is the kernel width.
The regularization parameter $\beta$ in Eq.~(\ref{mainss}) governs the trade-off between the reconstruction error $f(\mathbf{S})$ and the group sparsity $g(\mathbf{S})$ with respect to the selection matrix.
Specifically, by introducing the similarity score $h_i=\text{MBS}(\mathbf{x}_i,\text{Tmpl}_r)$ to Eq.~(\ref{mainss}), the selection matrix $\mathbf{S}$ is weighted by the similarity scores to faithfully represent the ``reliable" degrees of the corresponding tracking results.

\subsubsection{Optimization for memory filtering}
The objective function in Eq.~(\ref{mainss}) can be decomposed into a differentiable convex function $f(\mathbf{S})$ with \lchanged{M2.4}{\link{R2.4}}{a} Lipschitz continuous gradient and a non-smooth but convex function $g(\mathbf{S})$, so the accelerated proximal gradient (APG) \cite{Parikh2013Proximal_minis} algorithm can be used for efficiently solving this problem with the convergence rate of $\mathcal{O}(\frac{1}{T^2})$ ($T$ is the number of iterations).
Therefore, we need to solve the following optimization problem:
\begin{equation}\label{APG}
 \mathbf{Z}^{(t+1)}=\mathop{\mathrm{argmin}}\limits_{\mathbf{S}} \frac{1}{2}\|\mathbf{S}-\mathbf{Z}^{(t)}\|_\mathrm{F}^2+\frac{1}{p_L}g(\mathbf{S})\,,
\end{equation}
where the auxiliary variable $\mathbf{Z}=\mathbf{S}-\frac{1}{p_L}\nabla f(\mathbf{S})$,
and $p_L$ is the smallest feasible Lipschitz constant, which equals to:
\begin{equation}\label{Lips}
  p_L = \phi(\mathbf{X}^{\top}\mathbf{X}+\delta(\mathbf{L}+\mathbf{L}^{\top}))\,,
\end{equation}
where $\phi(\cdot)$ is the spectral radius of the corresponding matrix.
The gradient $\nabla f(\mathbf{S})$ is obtained by:
\begin{equation}\label{gradf}
 \nabla f(\mathbf{S})=-\mathbf{X}\mathbf{X}^{\top}+\mathbf{X}^{\top}\mathbf{X}\mathbf{S}
 +\delta (\mathbf{L}+\mathbf{L}^{\top})\mathbf{S}\,.
\end{equation}

Notice that the objective function in Eq.~(\ref{APG}) is separable regarding the rows of $\mathbf{S}$, thus we decompose Eq.~(\ref{mainss}) into \lchanged{M2.4}{\link{R2.4}}{an $N_s$ of group lasso sub-problems that can be effectively solved by a soft-thresholding operator}, which is:
\begin{equation}\label{soft}
  \mathbf{S}_{i,\cdot}\! \!=\!\!\mathbf{Z}_{i,\cdot}\mathop{\mathrm{max}}\! \bigg\{\!
  1-\frac{\frac{\beta}{p_L(h_i+C)}}{\|\mathbf{Z}_{i,\cdot}\|_2},0\bigg\},i=1,2,\cdots \!,\!N_s.
\end{equation}
Finally, the algorithm for the memory filtering strategy is summarized in Algorithm \ref{ago1}.

\begin{algorithm}[t]\label{ago1}
\begin{small}
\caption{Algorithm for memory filtering strategy}
\KwIn{data matrix $\mathbf{X} \in \mathbb{R}^{N_s\times d}$ with their corresponding similarity scores $\{ h_i \}_{i=1}^{N_s}$, two related regularization parameters: $\beta$, $\delta$.}
\KwOut{\changed{M2.4.4}{\link{R2.4}}{the selected representative result $\mathbf{x}_i$ with the largest value in $\| \mathbf{S}_{i,\cdot}\|_2$.}}
Set:  stopping error $\varepsilon$.\\
Obtain the Lipschitz constant $p_L$ by Eq.~(\ref{Lips}).\\
Initialize $t=0$ and $l^{(0)}=1$, $\mathbf{S}^{(0)}=\mathbf{0}$ and two auxiliary matrices $\mathbf{U}_1^{(0)}=\mathbf{U}_2^{(0)}=\mathbf{0}$.\\
\SetKwRepeat{RepeatUntil}{Repeat}{Until}
\RepeatUntil{$\frac{\| \mathbf{S}^{(t+1)} - \mathbf{S}^{(t)}\|_{1,2}} { \| \mathbf{S}^{i} \|_{1,2}} \leq \varepsilon $}
{$\mathbf{Z}^{(t+1)}:=\mathbf{U}_1^{(t)}-\frac{1}{p_L}\nabla f(\mathbf{U}_1^{(t)})$ by Eq.~(\ref{gradf})\;
$\mathbf{U}_2^{(t+1)}:=\mathbf{S}^{(t)}$ and $\mathbf{S}_i^{(t+1)}$ is obtained by Eq.~(\ref{soft}) for $i=1,2,\cdots,N_s$\;
$\tau:=l^{(t)}-1$, and $l^{(t+1)}:=\frac{1+\sqrt{1+(l^{(t)})^2}}{2}$\;
$\mathbf{U}_1^{(t+1)}:=\mathbf{S}^{(t+1)}
+\frac{\tau(\mathbf{S}^{(t+1)}-\mathbf{U}_2^{(t+1)})}{t}$\;
$t := t + 1$\;
}
Output $\mathbf{x}_i$ with the largest value in $\| \mathbf{S}_{i,\cdot}\|_2$.
\end{small}
\end{algorithm}

\subsubsection{Illustration of memory filtering for Tmpl$_r$}
In our tracker, the tracking results of the latest ten frames ($N_s=10$) are preserved to construct the matrix $\mathbf{X}$, and the $i$th ($i=1,2,\cdots,N_s$) tracking result $\mathbf{T}_i$ with the largest value $\| \mathbf{S}_{i,\cdot}\|_2$ is added into the target dictionary.
\lchanged{M1.5.20}{\link{R1.5}}{Furthermore, to save storage space and reduce computational cost, the ``First-in and First-out" procedure is employed to maintain the number of atoms $N_D$ in the target dictionary $\mathcal{D} \in \mathbb{R}^{d \times N_D}$.
That is, the latest representative tracking result is added, and meanwhile the oldest tracking result is thrown away.}

Here, similar to \cite{liu2017NMC_mini}, we detail how the template Tmpl$_r$ is represented by such \lchanged{M2.4}{\link{R2.4}}{a} carefully constructed dictionary.
\begin{figure}
\begin{center}
\includegraphics[width=0.44\textwidth]{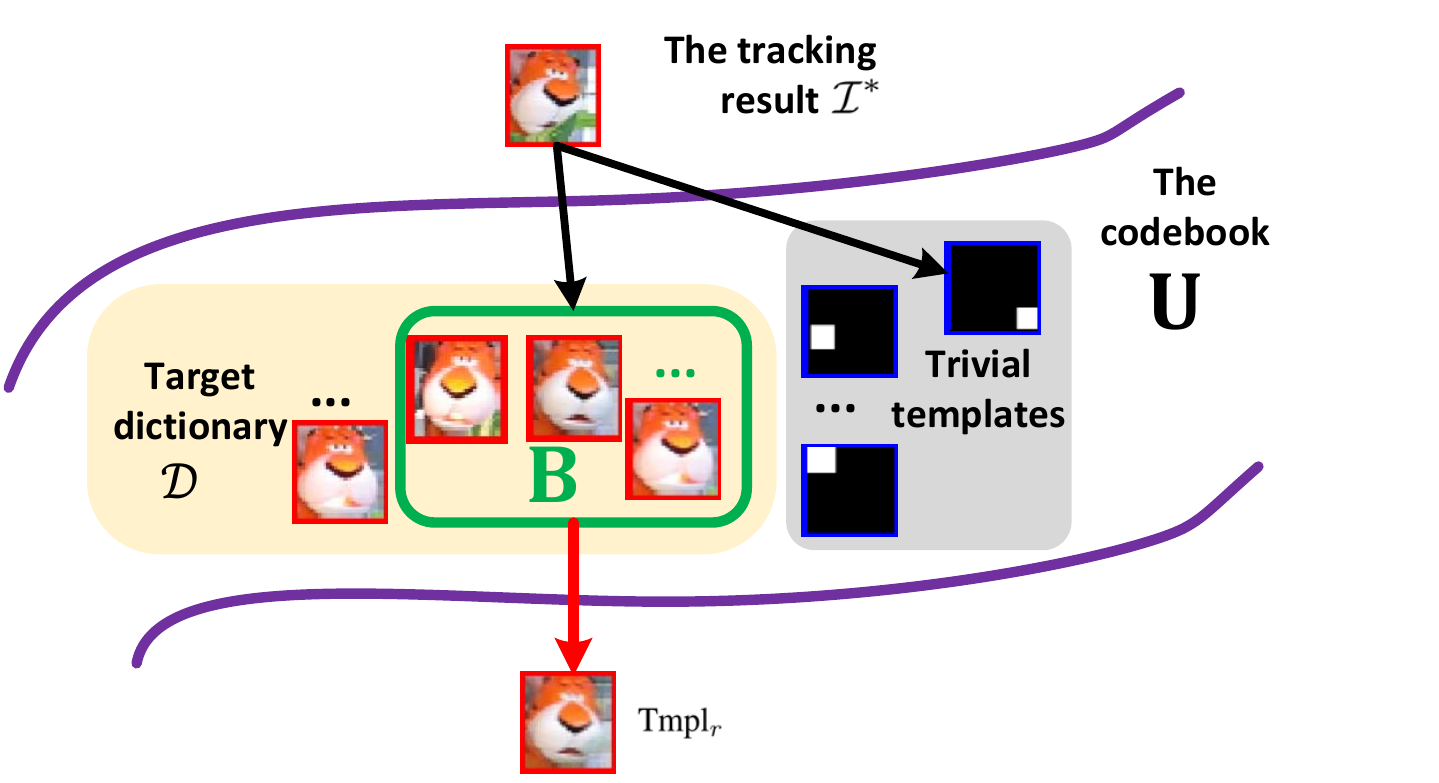}
\caption{\footnotesize Illustration of how the template Tmpl$_r$ is represented by the target dictionary $\mathcal{D}$. The tracking results are represented by its $k$ nearest neighbors from the target dictionary $\mathcal{D}$ and a certain amount of trivial templates.
Such selected target templates render the construction of the template Tmpl$_r$.}
\label{LLC}
\end{center}
\end{figure}
As shown in \lchanged{M1.5.21}{\link{R1.5}}{Fig.~\ref{LLC}}, in our method, we construct a codebook $\mathbf{U}=\mathcal{D} \cup \mathbf{I}$, where $\mathbf{I}$ represents a set of trivial templates\footnote{Each trivial template is formulated as a vector with only one nonzero element.}.
Subsequently, we select $k$ ($k=5$ in our experiment) nearest neighbors of the tracking result $\mathcal{I}^*$ from the codebook $\mathbf{U}$, to form the dictionary $\mathbf{B}$.
Finally, the template Tmpl$_r$ is reconstructed by a linear combination of atoms in the dictionary $\mathbf{B}$.
By doing so, the appearance model can effectively avoid being contaminated when the tracking result $\mathcal{I}^*$ is slightly occluded.
Fig.~\ref{LLC} demonstrates that the template Tmpl$_r$ is much more accurate than $\mathcal{I}^*$ because the leaf in front of the target is removed in the template Tmpl$_r$.

To show the effectiveness of our memory filtering strategy, a qualitative result is shown in \changed{M1.5.23}{\link{R1.5}}{Fig.~\ref{key}}.
\begin{figure}
\begin{center}
\includegraphics[width=0.45\textwidth]{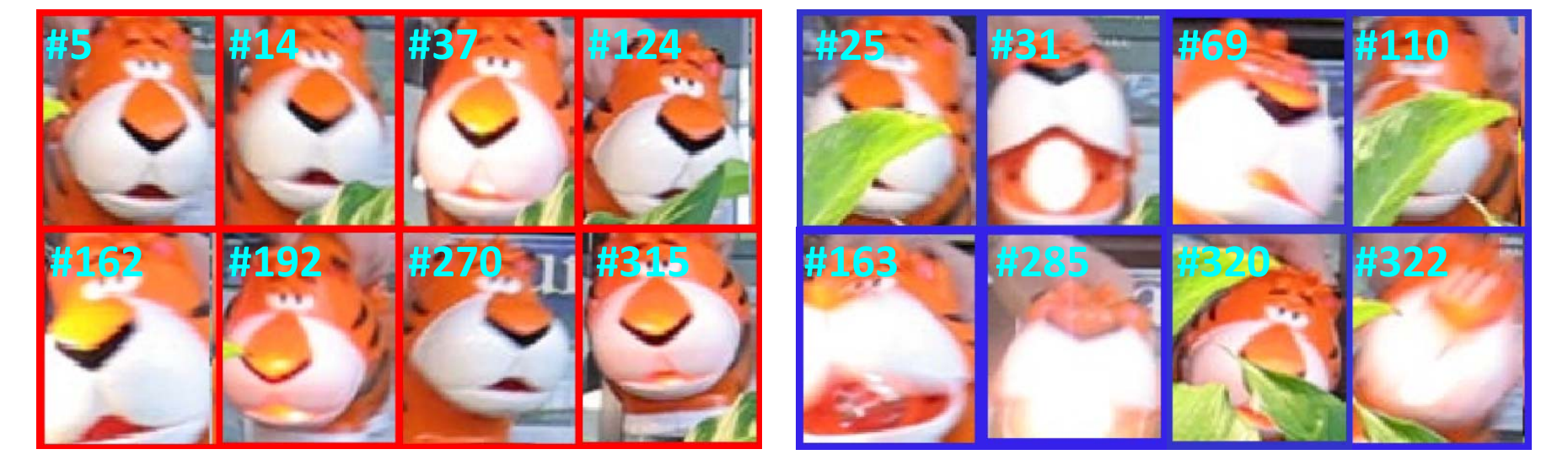}
\caption{\footnotesize Examples of some selected (red boxes) and discarded (blue boxes) historical tracking results by our memory filtering strategy.  We see that the selected results represent the general appearance of the target, so they are reliable and should be ``recalled". The results under abnormal situations (\emph{e.g.} occlusion, incompleteness, and undesirable observed angle) are filtered out and ``forgotten" by the memory filtering strategy.}
\label{key}
\end{center}
\end{figure}
One can see that the memory filtering strategy selects some representative and reliable results (in red), which depict the target appearance in normal conditions, so they can precisely represent the target in general cases.
\lchanged{M1.5.22}{\link{R1.5}}{Comparably, some tracking results under drastic appearance changes, severe occlusions and dramatic illumination variations are not incorporated to the template Tmpl$_r$.
This is because these results just temporarily present the abnormal situations of the target, \emph{i.e.}, far away from the target's general appearances.}
Note that these \lchanged{M1.5.24}{\link{R1.5}}{dramatic} appearance variations in a short period can be captured by the template Tmpl$_e$.
Therefore, the combination of \lchanged{M1.5.25}{\link{R1.5}}{two such types of templates effectively decreases the risk of tracking drifts, so that the appearance of interesting} target can be comprehensively understood by our TM$^3$ tracker.

Finally, our TM$^3$ tracker is summarized in Algorithm \ref{ago2}.

\section{Implementation Details}
\label{sec:impl}
In this section, more implementation details of our method will be discussed.

\subsection{Geometry Constraint}
\label{sec:screen}
The \emph{fast candidate selection} step aims at designing a fast algorithm to throw away some definitive non-target regions  to balance between running speed and accuracy.
The obtained region $\mathcal{I}_r^*$ at the $t$th frame in \texttt{Flow-r} process helps to remove numerous definitive non-target regions in \texttt{Flow-e} process.
To this end, we introduce a distance measure $dist$ \cite{Bailer2014A_mini} between the $j$th target proposal $\mathcal{I}(\mathbf{c}_e^j,s_e^j)$ ($\mathcal{I}_e^j$ for simplicity) at the $t$th frame and $\mathcal{I}(\mathbf{c}_r^*,s_r^*)$, \lchanged{M1.5.27}{\link{R1.5}}{that is}:
\begin{equation}\label{dist}
\begin{split}
&  dist([\mathbf{c}_e^j,s_e^j],[\mathbf{c}_r^*,s_r^*])= \\
&\| \frac{c_r^{*x}-c_e^{jx}}{w(s_r^*+s_e^j)}, \frac{c_r^{*y}-c_e^{jy}}{h(s_r^*+s_e^j)},
  \tau \frac{s_r^*-s_e^j}{s_r^*+s_e^j} \|_2,
\end{split}
\end{equation}
\lchanged{M2.4.3}{\link{R2.4}}{where $\tau=5$ is a parameter determining the influence of scale to the value of $dist$, and $w$, $h$ are the width and height of the final tracking result at the $(t-1)$th frame, respectively.}
A small $dist$ value indicates that the corresponding image region $\mathcal{I}_e^j$ is very similar to the tracking cue $\mathcal{I}_r^*$ with a high probability.
Following the definition of such distance, in our method the top $N'_e$ target proposals with the smallest $dist$ value are retained, so that they can be used in the following step in \texttt{Flow-e} process.
Therein, two tracking cues including $\mathcal{I}_e^{dist}$ with the smallest $dist$, and the $\mathcal{I}_e^*$ with the highest similarity score are picked up for the fusion step.

\subsection{Fusion of Multiple Tracking Cues}
\label{sec:fuse}
Three tracking cues $\mathcal{I}_r^*$, $\mathcal{I}_e^*$ and $\mathcal{I}_e^{dist}$ are obtained by two main flows as shown in Fig.~\ref{framework}.
Here we fuse the above results to the final tracking result based on a confidence level $F$.
For example, the confidence level of $\mathcal{I}_r^*$ is defined as:
\begin{equation}\label{fusion}
\begin{split}
&F(\mathcal{I}_r^*)\!=\!\text{MBS}(\mathcal{I}_r^*,\text{Tmpl}_r)
+\!\text{MBS}(\mathcal{I}_r^*,\text{Tmpl}_e)\\
&+\text{VOR}([\mathbf{c}_r^*,s_r^*],[\mathbf{c}_e^{dist},s_e^{dist}])
+\text{VOR}([\mathbf{c}_r^*,s_r^*],[\mathbf{c}_e^*,s_e^*]),
\end{split}
\end{equation}
where the Pascal VOC Overlap Ratio (VOR) \cite{VOC_mini} measures the overlap rate between the two bounding boxes $A$ and $B$, namely $\text{VOR}(A,B)=\frac{area(A \cap B)}{area(A \cup B)}$.
The first two terms in Eq.~(\ref{fusion}) reflect the appearance similarity degree and the last two terms consider the spatial relationship of the two cues.
The confidence levels for $F(\mathcal{I}_e^*)$ and $F(\mathcal{I}_e^{dist})$ can be calculated in the similar way.
Finally, the tracking cue with the highest confidence level is chosen as the final tracking result $\mathcal{I}^*$.

\begin{algorithm}
\label{ago2}
\begin{small}
\caption{The proposed TM$^3$ tracking algorithm}
\KwIn{Initial target bounding box $\mathbf{o}_1=(x_1, y_1, s_1)$.}
\KwOut{Estimated object state $\mathbf{o}_t^*=(\hat{x}_t, \hat{y}_t, \hat{s}_t)$.}
\SetKwRepeat{RepeatUntil}{Repeat}{Until}
\RepeatUntil{End of video sequence}
{Generate candidate target regions $\mathcal{R}_t \bigcup \mathcal{E}_t$ \;
\tcp*[h]{Flow-r process}\\
Obtain potential candidate regions  $\mathcal{R}'_t$\;
MBS: Find the optimal region $\mathcal{I}_r^*$ from $\mathcal{R}'_t$\;
\tcp*[h]{Flow-e process}\\
Obtain potential candidate regions $\mathcal{E}'_t$ by Eq.~(\ref{dist})\;
MBS: Obtain $\mathcal{I}_e^*$ and $\mathcal{I}_e^{dist}$ from $\mathcal{E}'_t$\;
Output $\mathbf{o}_t^*$ and $\mathcal{I}^*$ by the fusion step in Eq.~(\ref{fusion})\;
\tcp*[h]{Update $\text{Tmpl}_e$}\\
\lIf{$\text{MBS}(\mathcal{I}^*,\text{Tmpl}_e)>0.5$}{$\text{Tmpl}_e \leftarrow \mathcal{I}^*$\;}
\tcp*[h]{Select the representative result}\\
\lIf{$t$ $\mathrm{mod}$ $10 = 0$}{Obtain $\mathbf{T}_i$ by Algorithm \ref{ago1}\;}
\tcp*[h]{Update $\mathcal{D}$ and $\text{Tmpl}_r$}\\
\lIf{$\text{Num}(\mathcal{D})=N_D$}{Obtain Tmpl$_r$ by $\mathcal{D}$ \;}
\ElseIf{$\text{Num}(\mathcal{D})<N_D$}{$\mathcal{D} = \mathcal{D} \cup \mathbf{T}_i$, $\text{Num}(\mathcal{D}):=\text{Num}(\mathcal{D})+1$\;}
\lElse{``First-in and First-out" procedure for $\mathcal{D}$\;}

}
\end{small}
\end{algorithm}

\subsection{Feature Descriptors}
We \changed{M2.4}{\link{R2.4}}{experimented with two appearance descriptors consisting of} color feature and deep feature to represent the target regions and the templates.

\noindent {\bf Color features:}
For a colored video sequence, all target regions and the templates are normalized into $36 \times 36 \times 3$ in CIE Lab color space.
In \changed{M2.4}{\link{R2.4}}{the} \texttt{Flow-e} process, the EdgeBox approach is executed on RGB color space to generate various target proposals $\mathcal{E}_t$.
In \changed{M1.5.28}{\link{R1.5}}{the two} processes, each image region is split into $3 \times 3$ non-overlapped small patches, \changed{M2.4}{\link{R2.4}}{where each patch is represented by a 27-dimensional ($3 \times 3 \times 3$) feature vector.}

\noindent {\bf Deep features:}
We adopt the Fast R-CNN \cite{Girshick2015Fast_mini} with a pre-trained VGG16 model on ImageNet \cite{Deng2009ImageNet_mini} and PASCAL07 \cite{VOC_mini} for our region-based feature extraction.
In our method, the Fast R-CNN network takes the entire image $\mathcal{I}$ and the potential candidate proposals $\mathcal{R}'_t\cup\mathcal{E}'_t$ as input.
For each candidate proposal, \emph{the region of interest (ROI) pooling layer} in the network architecture is adopted to exact a 4096-dimensional feature vector from the feature map to represent each image region.

\subsection{Computational Complexity of MBS}
\label{sec:ccm}
\changed{M2.1}{\link{R2.1}}{
The computation of MBS can be divided into two steps: first to calculate the similarity matrix, and second to pick up $r$ (or $s$) reciprocal nearest neighbors of each image patch based on the similarity matrix as demonstrated in Eq.~\eqref{mBBPdef}.

Thanks to the reciprocal $k$-NN scheme, the generated similarity matrix is sparse and the nonzero elements in the matrix almost spread along its diagonal direction.
As a result, the average computational complexity of the similarity matrix reduces from $\mathcal{O}(M^2d)$ to $\mathcal{O}(Md)$, where $d$ is the feature dimension.
After that, we pick up $r$ (or $s$) reciprocal nearest neighbors of each image patch in $\mathcal{P}$ based on the similarity matrix.
Due to the exponential decay operator in Eq.~\eqref{mBBPdef}, there is no sense to consider a large $r$ and $s$.
Hence, we just consider the $c\leq 4$ nearest neighbors of an image patch to accelerate the computation in our experiment.
As described in \cite{Dekel_BBS_mini}, such operation can be further reduced to $\mathcal{O}(Mc)$ on the average.
Finally, the overall MBS complexity is $\mathcal{O}(M^2cd)$.
}
\section{Experiments}
\label{sec:experiment}
In this section, we compare the proposed TM$^3$ tracker with other recent algorithms on two benchmarks including OTB \cite{Wu2015_mini} and PTB \cite{song2013tracking_mini}.
Moreover, the results of ablation study and parametric sensitivity are also provided.

\subsection{Experimental Setup}
In our experiment, the proposed TM$^3$ tracker is tested on both color feature (denoted as ``TM$^3$-color") and deep feature (denoted as TM$^3$-deep).
Our TM$^3$-color tracker is implemented in MATLAB on a
PC with Intel i5-6500 CPU (3.20 GHz) and 8 GB memory, and runs about 5 fps (frames per second).
\lchanged{M3.11}{\link{R3.11}}{The proposed TM$^3$-deep tracker is based on MatConvNet toolbox \cite{Vedaldi2015MatConvNet_mini} with Intel Xeon E5-2620 CPU @2.10GHz and a NVIDIA GTX1080 GPU, and runs almost 4 fps.}

\noindent{\bf Parameter settings}
In our TM$^3$-color tracker, every image region is normalized to $36\times 36$ pixels, and then split into a set of non-overlapped $3\times 3$ image patches.
In this case, $M=N=\frac{{36}^2}{3^2}=144$.
In the TM$^3$-deep tracker, each image region is represented by a 4096-dimensional feature vector, namely $M=N=4096$.
 In \texttt{Flow-r} process, to generate target regions in the $t$th frame, we draw $N_r=700$ samples in translation and scale dimension, $\mathbf{x}_i^t = (\mathbf{c}_r^i,s_r^i), i=1,2,\cdots,N_r$, from a Gaussian distribution whose mean is the previous target state $\mathbf{x}_*^{t-1}$ and covariance is a diagonal matrix $\Sigma=\mathrm{diag}(\sigma_x,\sigma_y,\sigma_s)$ of which diagonal elements are standard deviations of the sampling parameter vector $[\sigma_x,\sigma_y,\sigma_s]$ for representing the target state.
In our experiments, the sampling parameter vector is set to $[\sigma_x,\sigma_y,0.15]$ where $\sigma_x=\min\{w/4,15\}$ and $\sigma_y=\min\{h/4,15\}$ are fixed for all test sequences, and $w$, $h$ have been defined in Eq.~\eqref{dist}.
The number of potential proposals $\mathcal{R}'_t$ and $\mathcal{E}'_t$ are set to $N'_r=N'_e=50$.
In \texttt{Flow-e} process, we use the same parameters in EdgeBox as described in \cite{zitnick2014edge_mini}.
The trade-off parameters $\delta$ and $\beta$ in Eq.~(\ref{mainss}) are fixed to $5$ and $10$ accordingly;
The number of atoms in the target dictionary $\mathcal{D}$ is decided as $N_D=12$.

\subsection{Results on OTB}
\subsubsection{Dataset description and evaluation protocols}
\label{ddde}
OTB includes two versions, \emph{i.e.} OTB-2013 and OTB-2015.
OTB-2013 contains 51 sequences with precise bounding-box annotations, and 36 of them are colored sequences.
In OTB-2015, there are 77 colored video sequences among all the 100 sequences.
Specifically, considering that the proposed TM$^3$-color tracker is executed on CIE Lab and RGB color spaces, the TM$^3$-color tracker is only compared with other baseline trackers on the colored sequences to achieve fair comparison.
Differently, our TM$^3$-deep tracker can handle both colored and gray-level sequences, so it is evaluated on all the sequences in the above two benchmarks.

The quantitative analysis on OTB is demonstrated on two evaluation plots in the one-pass evaluation (OPE) protocol: the success plot and the precision plot.
In the success plot, the target in a frame is declared to be successfully tracked if its current overlap rate exceeds a certain threshold.
The success plot shows the percentage of successful frames at the overlap threshold varies from 0 to 1.
In the precision plot,  the tracking result in a frame is considered successful if the center location error (CLE) falls below a pre-defined threshold.
The precision plot shows the ratio of successful frames at the CLE threshold ranging from 0 to 50.
Based on the above two evaluation plots, two ranking metrics are used to evaluate all compared trackers: one is the Area Under the Curve (AUC) metric for the success plot, and the other is the precision score at threshold of 20 pixels for the precision plot.
For details about the OTB protocol, refer to \cite{Wu2015_mini}.

Apart from the totally 29 and 37 trackers included in OTB-2013 and OTB-2015, respectively, we also compare our tracker with several state-of-the-art methods, including ACFN \cite{Choi2017CVPR_mini}, RaF \cite{Zhangle2017CVPR_mini}, TrSSI-TDT \cite{Hu2017Semi_mini}, DLSSVM \cite{Ning2016CVPR_mini}, Staple \cite{Bertinetto2016CVPR_mini}, DST \cite{DSTXiao2016Distractor_mini}, DSST \cite{danelljan2014accurate_mini}, MEEM \cite{zhang2014meem_mini}, TGPR \cite{Gao2014_mini}, KCF \cite{henriques2015high_mini}, IMT \cite{yoon2015interacting_mini}, LNLT \cite{Mabo2015_mini}, DAT \cite{DATPossegger2015In_mini}, and CNT \cite{zhang2016cnt_mini}\footnote{The implementation of several algorithms \emph{i.e.}, TrSSI-TDT and LNLT is not public, and hence we just report their results on OTB provided by the authors for fair comparisons.}.
Specifically, two state-of-the-art template matching based trackers including ELK \cite{Oron2014klt_mini} and BBT \cite{Oron2016Best_mini} are also incorporated for comparison.

\subsubsection{Overall performance}


\begin{figure*}
\begin{center}
\includegraphics[width=0.96\textwidth]{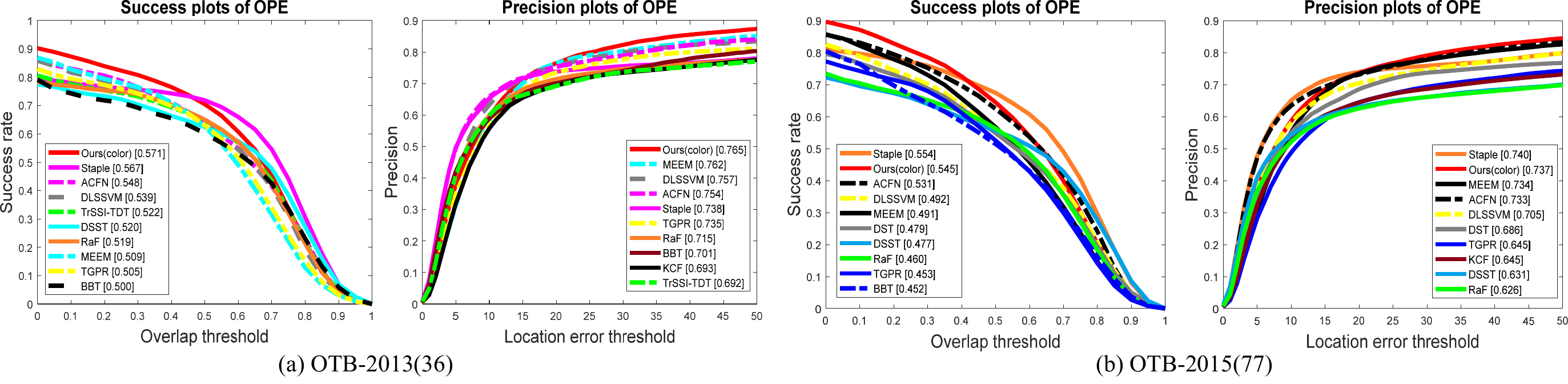}
\caption{\footnotesize Success  and precision plots of our color-based tracker TM$^3$-color and various compared trackers. (a) shows the results on OTB-2013 with 36 colored sequences, and (b) presents the results on OTB-2015 with 77 colored sequences. For clarity, we only show the curves of top 10 trackers in (a) and (b).}
\label{OPEcolor}
\end{center}
\end{figure*}

Fig.~\ref{OPEcolor} shows the performance of all compared trackers on OTB-2013 and OTB-2015 datasets.
On OTB-2013, our TM$^3$ tracker with color feature achieves 57.1\% on average overlap rate, which is higher than the 56.7\% produced by a very competitive correlation filter based algorithm Staple.
On OTB-2015, it can be observed that the performance of all trackers \changed{M1.5.29}{\link{R1.5}}{decreases}.
The proposed TM$^3$-color tracker and Staple still provide the best results with the AUC scores equivalent to 54.5\% and 55.4\%, respectively.
Specifically, on these two benchmarks, we see that the competitive template matching based tracker BBT obtains 50.0\% and 45.2\% on average overlap rate, respectively.
Comparatively, our TM$^3$ tracker significantly improves the performance of BBT with a noticeable margin of 7.1\% and 9.3\% on OTB-2013 and OTB-2015, accordingly.

\begin{figure*}
\begin{center}
\includegraphics[width=0.96\textwidth]{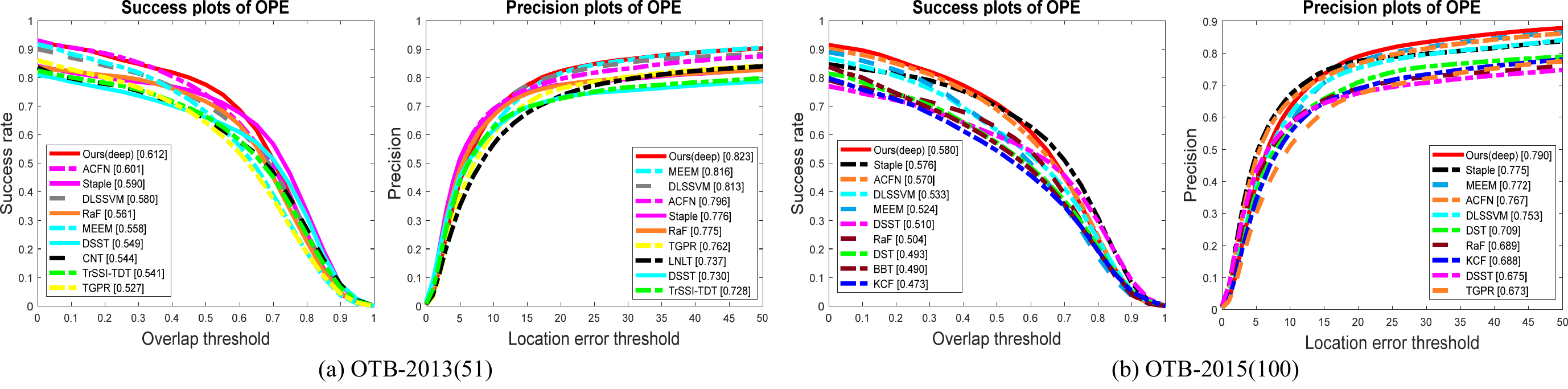}
\caption{\footnotesize Success  and precision plots of our deep feature based tracker TM$^3$-deep and various compared trackers. (a) shows the results on OTB-2013 with 51 sequences, and  (b) presents the results on OTB-2015 containing 100 sequences. For clarity, we only show the curves of top 10 trackers.}
\label{OPEdeep}
\end{center}
\end{figure*}

We also test our tracker with deep feature and the corresponding performance of these trackers are shown in Fig.~\ref{OPEdeep}.
Not surprisingly, TM$^3$-deep tracker boosts the
performance of TM$^3$-color with color feature.
It achieves 61.2\% and 58.0\% success rates on the above two benchmarks, both of which rank first among all compared trackers.
On the precision plots, the proposed TM$^3$-deep tracker yields the precision rates of 82.3\% and 79.0\% on the two benchmarks, respectively.

\lchanged{M3.5}{\link{R3.5}}{The overall plots on the two benchmarks demonstrate that our TM$^3$ (with colored and deep features) tracker comes in first or second place among the trackers with a comparable performance evaluated by the success rate.}
It is able to outperform the trackers such as CNN based trackers, correlation filter based algorithms, template matching based approaches, and other representative methods.
The favorable performance of our TM$^3$ tracker benefits from the fact that
the discriminative similarity metric, the memory filtering strategy, and the rich feature help our TM$^3$ tracker to accurately separate the target from its cluttered background, and effectively capture the target appearance variations.

\subsubsection{Attribute based performance analysis}
\begin{figure*}
\begin{center}
\includegraphics[width=0.96\textwidth]{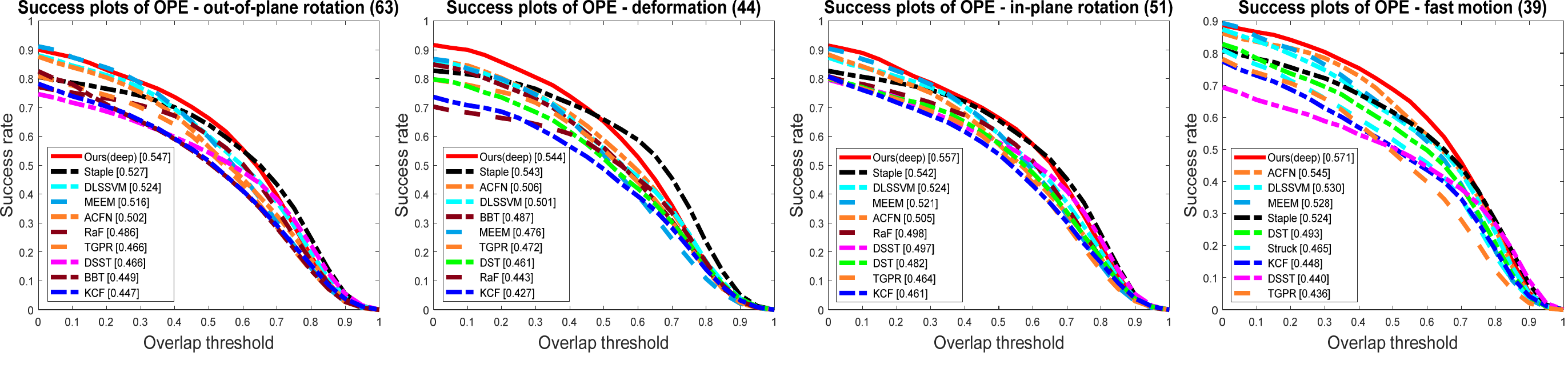}
\caption{\footnotesize Attribute-based analysis of our TM$^3$-deep tracker with four main attribute on OTB-2015(100), respectively. For clarity, we only show the top 10 trackers in the legends. The title of each plot indicates the number of videos labelled with the respective attribute.}
\label{attribute}
\end{center}
\end{figure*}

To analyze the strength and weakness of the proposed algorithm, we provide the attribute based performance analysis to illustrate the superiority of our tracker on four key attributes in Fig.~\ref{attribute}.
All video sequences in OTB have been manually annotated with several challenging attributes, including \emph{Occlusion} (OCC), \emph{Illumination Variation} (IV), \emph{Scale Variation} (SV), \emph{Deformation} (DEF), \emph{Motion Blur} (MB), \emph{Fast Motion} (FM), \emph{In-Plane Rotation}, \emph{Out-of-Plane Rotation} (OPR), \emph{Out-of-View} (OV), \emph{Background Clutter} (BC), and \emph{Low Resolution} (LR).
\lchanged{M1.5.30}{\link{R1.5}}{As illustrated in Fig.~\ref{attribute}, our TM$^3$-deep tracker performs the best on OPR, DEF, IPR, and FM attributes when compared to some representative trackers.}
The favorable performance of our tracker on appearance variations (\emph{e.g.} OPR, IPR, and DEF) demonstrates the effectiveness of the discriminative similarity metric and the memory filtering strategy.
\subsection{Results on PTB}
The PTB benchmark database contains 100 video sequences with both RGB and depth data under highly diverse circumstances.
These sequences are grouped into the following aspects: target type (human, animal and rigid), target size (large and small), movement (slow and fast), presence of occlusion, and motion type (passive and active).
Generally, the human and animal targets including dogs and rabbits often suffer from out-of-plane rotation and severe deformation.

\changed{M3.6}{\link{R3.6}}{In PTB evaluation system, the ground truth of only 5 video sequences is shared for parameter tuning.
Meanwhile, the author make the ground truth of the remaining 95 video sequences inaccessible to public for fair comparison.
Hence, the compared algorithms, conducted on these 95 sequences, are allowed to submit their tracking results for performance comparison by an online evaluation server.}
Hence, the benchmark is fair and valuable in evaluating the effectiveness of different tracking algorithms.
Apart from 9 algorithms using RGB data included in PTB, we also compare the proposed tracker with 12 recent algorithms appeared
in Section \ref{ddde}.
 Tab.~\ref{tabptb} shows the average overlap ratio and ranking results of these compared trackers on 95 sequences.
 The top five trackers are TM$^3$-deep, TM$^3$-color, Staple, ACFN, and RaF.
 The results show that the proposed TM$^3$-deep tracker again achieves the state-of-the-art performance over other trackers.
 Specifically, it is worthwhile to mention that our method performs better than CFTs on large appearance variations (\emph{e.g.}, human and animal) and fast movement.

\renewcommand\arraystretch{1.0}
\begin{table*}[t]
\label{tabptb}
\centering
\begin{center}
\scriptsize
\caption{\footnotesize Results on the Princeton Tracking Benchmark with 95 video sequences: success rates and rankings (in parentheses) under different sequence categorizations. The best three results are highlighted by \textcolor[rgb]{1,0,0}{red}, \textcolor[rgb]{0,0,1}{blue}, and \textcolor[rgb]{0,1,0}{green}, respectively.}\label{tabptb}
\begin{tabular}{|p{1.5cm}|p{1.0cm}|p{0.9cm}|p{0.9cm}|p{0.9cm}|p{0.9cm}|p{0.9cm}|p{0.9cm}|p{0.9cm}|p{0.9cm}|p{0.9cm}|p{0.9cm}|p{0.9cm}|}
  \hline
  \multirow{2}{*}{Method} &{Avg.} &\multicolumn{3}{c|}{target type} &\multicolumn{2}{c|}{target size} &\multicolumn{2}{c|}{movement} &\multicolumn{2}{c|}{occlusion} &\multicolumn{2}{c|}{motion type} \\
  \cline{3-13}
  \cline{4-5}
  &Rank &human &animal &rigid &large &small &slow &fast &yes &no &passive &active \\
  \hline
TM$^3$-deep  &{ \textcolor[rgb]{1.00,0.00,0.00}{2.364(1)}} &\textcolor[rgb]{1.00,0.00,0.00}{0.612(1)} &\textcolor[rgb]{1.00,0.00,0.00}{0.672(1)} &\textcolor[rgb]{1.00,0.00,0.00}{0.691(1)} &\textcolor[rgb]{1.00,0.00,0.00}{0.586(1)} &0.502(11) &\textcolor[rgb]{1.00,0.00,0.00}{0.724(1)} &\textcolor[rgb]{1.00,0.00,0.00}{0.625(1)} &\textcolor[rgb]{0.00,0.00,1.00}{0.526(2)} &0.692(4) &\textcolor[rgb]{1.00,0.00,0.00}{0.701(1)} &\textcolor[rgb]{0.00,0.00,1.00}{0.551(2)} \\
TM$^3$-color &{\textcolor[rgb]{0.00,0.00,1.00}{3.818(2)}} &0.551(4) &\textcolor[rgb]{0.00,0.00,1.00}{0.657(2)} &0.547(10) &0.535(4) &0.513(9) &\textcolor[rgb]{0.00,0.00,1.00}{0.683(2)} &\textcolor[rgb]{0.00,0.00,1.00}{0.597(2)} &\textcolor[rgb]{0.00,1.00,0.00}{0.511(3)} &\textcolor[rgb]{0.00,0.00,1.00}{0.695(2)} &\textcolor[rgb]{0.00,1.00,0.00}{0.646(3)} &\textcolor[rgb]{1.00,0.00,0.00}{0.559(1)} \\
 Staple \cite{Bertinetto2016CVPR_mini} 	 &\textcolor[rgb]{0.00,1.00,0.00}{4.909(3)} &0.529(5) &\textcolor[rgb]{0.00,1.00,0.00}{0.619(3)} &0.553(8) &\textcolor[rgb]{0.00,1.00,0.00}{0.555(3)} &0.556(4) &0.652(4) &0.514(4) &0.455(9) &0.690(5) &0.631(5) &0.524(4) \\
 ACFN  \cite{Choi2017CVPR_mini}	 &5.182(4) &\textcolor[rgb]{0.00,0.00,1.00}{0.574(2)} &0.538(8) &\textcolor[rgb]{0.00,1.00,0.00}{0.599(3)} &0.505(8) &\textcolor[rgb]{0.00,1.00,0.00}{0.557(3)} &\textcolor[rgb]{0.00,1.00,0.00}{0.653(3)} &0.504(5) &0.482(7) &0.655(7) &0.603(8) &\textcolor[rgb]{0.00,1.00,0.00}{0.535(3)} \\
 RaF \cite{Zhangle2017CVPR_mini}	 &5.818(5) &\textcolor[rgb]{0.00,1.00,0.00}{0.572(3)} &0.542(7) &0.557(5) &0.515(7) &0.527(7) &0.582(10) &0.498(6) &0.492(5) &\textcolor[rgb]{1.00,0.00,0.00}{0.706(1)} &0.604(7) &0.483(6) \\
 DLSSVM \cite{Ning2016CVPR_mini}	 &6.455(6) &0.522(6) &0.584(4) &0.523(16) &\textcolor[rgb]{0.00,0.00,1.00}{0.563(2)} &\textcolor[rgb]{0.00,0.00,1.00}{0.559(2)} &0.597(6) &0.455(10) &0.458(8) &\textcolor[rgb]{0.00,1.00,0.00}{0.694(3)} &\textcolor[rgb]{0.00,0.00,1.00}{0.658(2)} &0.433(12) \\
 MEEM	\cite{zhang2014meem_mini} &7.455(7) &0.477(10) &0.510(10) &0.556(6) &0.523(5) &\textcolor[rgb]{1.00,0.00,0.00}{0.587(1)} &0.610(5) &0.436(12) &0.433(12) &0.644(9) &0.638(4) &0.458(8) \\
 KCF	\cite{henriques2015high_mini} &7.636(8) &0.464(11) &0.519(9) &0.594(4) &0.491(9) &0.547(5) &0.594(7) &0.494(7) &0.417(13) &0.668(6) &0.627(6) &0.480(7) \\
 BBT	\cite{Oron2016Best_mini} &9.091(9) &0.422(14) &0.553(5) &\textcolor[rgb]{0.00,0.00,1.00}{0.610(2)} &0.452(13) &0.511(10) &0.583(9) &\textcolor[rgb]{0.00,1.00,0.00}{0.521(3)} &0.448(10) &0.572(15) &0.575(9) &0.451(10) \\
 DSST \cite{danelljan2014accurate_mini}	 &9.909(10) &0.512(7) &0.551(6) &0.472(19) &0.480(10) &0.516(8) &0.591(8) &0.471(8) &0.408(15) &0.651(8) &0.561(11) &0.458(9) \\
 TGPR \cite{Gao2014_mini}	 &11.182(11) &0.484(8) &0.466(16) &0.498(18) &0.519(6) &0.530(6) &0.535(14) &0.459(9) &0.445(11) &0.611(13) &0.521(17) &0.504(5) \\
 DST \cite{DSTXiao2016Distractor_mini}	 &11.909(12) &0.436(12) &0.495(11) &0.554(7) &0.416(16) &0.467(12) &0.522(17) &0.413(14) &\textcolor[rgb]{1.00,0.00,0.00}{0.546(1)} &0.630(12) &0.546(14) &0.415(15) \\
 DAT \cite{DATPossegger2015In_mini} &12.364(13) &0.483(9) &0.484(13) &0.545(12) &0.473(12) &0.440(17) &0.543(13) &0.425(13) &0.495(4) &0.577(14) &0.521(18) &0.437(11) \\
 CNT \cite{zhang2016cnt_mini}	 &13.909(14) &0.424(13) &0.455(18) &0.551(9) &0.475(11) &0.459(15) &0.533(15) &0.377(17) &0.484(6) &0.563(16) &0.495(20) &0.421(13) \\
 Struck 	\cite{hare2011_mini} &14.909(15) &0.354(16) &0.470(14) &0.534(15) &0.450(14) &0.439(18) &0.580(11) &0.390(16) &0.304(19) &0.635(10) &0.544(15) &0.406(16) \\
 IMT \cite{yoon2015interacting_mini} &15.000(16) &0.324(17) &0.457(17) &0.545(13) &0.425(15) &0.444(16) &0.530(16) &0.445(11) &0.364(16) &0.536(18) &0.557(12) &0.418(14) \\
 VTD \cite{Kwon2010visual_mini}	 &15.273(17) &0.309(20) &0.488(12) &0.539(14) &0.386(18) &0.462(13) &0.573(12) &0.372(18) &0.283(20) &0.631(11) &0.549(13) &0.385(17) \\
	RGBdet \cite{song2013tracking_mini} 	 &17.636(18) &0.267(22) &0.409(20) &0.547(11) &0.319(22) &0.460(14) &0.505(20) &0.357(19) &0.348(17) &0.468(20) &0.562(10) &0.342(19) \\
ELK \cite{Oron2014klt_mini} 	 &17.727(19) &0.386(15) &0.434(19) &0.502(17) &0.352(20) &0.368(20) &0.514(19) &0.395(15) &0.416(14) &0.347(22) &0.528(16) &0.369(18) \\
	CT \cite{zhang2014ct_mini}   	 &19.727(20) &0.311(19) &0.467(15) &0.369(22) &0.390(17) &0.344(22) &0.486(21) &0.315(20) &0.233(23) &0.543(17) &0.421(21) &0.342(20) \\
	TLD \cite{kalal2012tracking_mini}  &20.273(21) &0.290(21) &0.351(22) &0.444(20) &0.325(21) &0.385(19) &0.516(18) &0.297(22) &0.338(18) &0.387(21) &0.502(19) &0.305(22) \\
	MIL \cite{Babenko2011_mini}  &20.636(22) &0.322(18) &0.372(21) &0.383(21) &0.366(19) &0.346(21) &0.455(22) &0.315(21) &0.256(21) &0.490(19) &0.404(23) &0.336(21) \\
	SemiB \cite{Grabner2008_mini}    &22.818(23) &0.225(23) &0.330(23) &0.327(23) &0.240(23) &0.316(23) &0.382(23) &0.244(23) &0.251(22) &0.327(23) &0.419(22) &0.232(23) \\
OF \cite{song2013tracking_mini}  	 &24.000(24) &0.179(24) &0.114(24) &0.234(24) &0.201(24) &0.175(24) &0.181(24) &0.188(24) &0.159(24) &0.223(24) &0.234(24) &0.168(24) \\
 \hline
\end{tabular}
\end{center}
\end{table*}

\subsection{Ablation Study and Parameter Sensitivity Analysis}
\begin{figure}
\begin{center}
\includegraphics[width=0.48\textwidth]{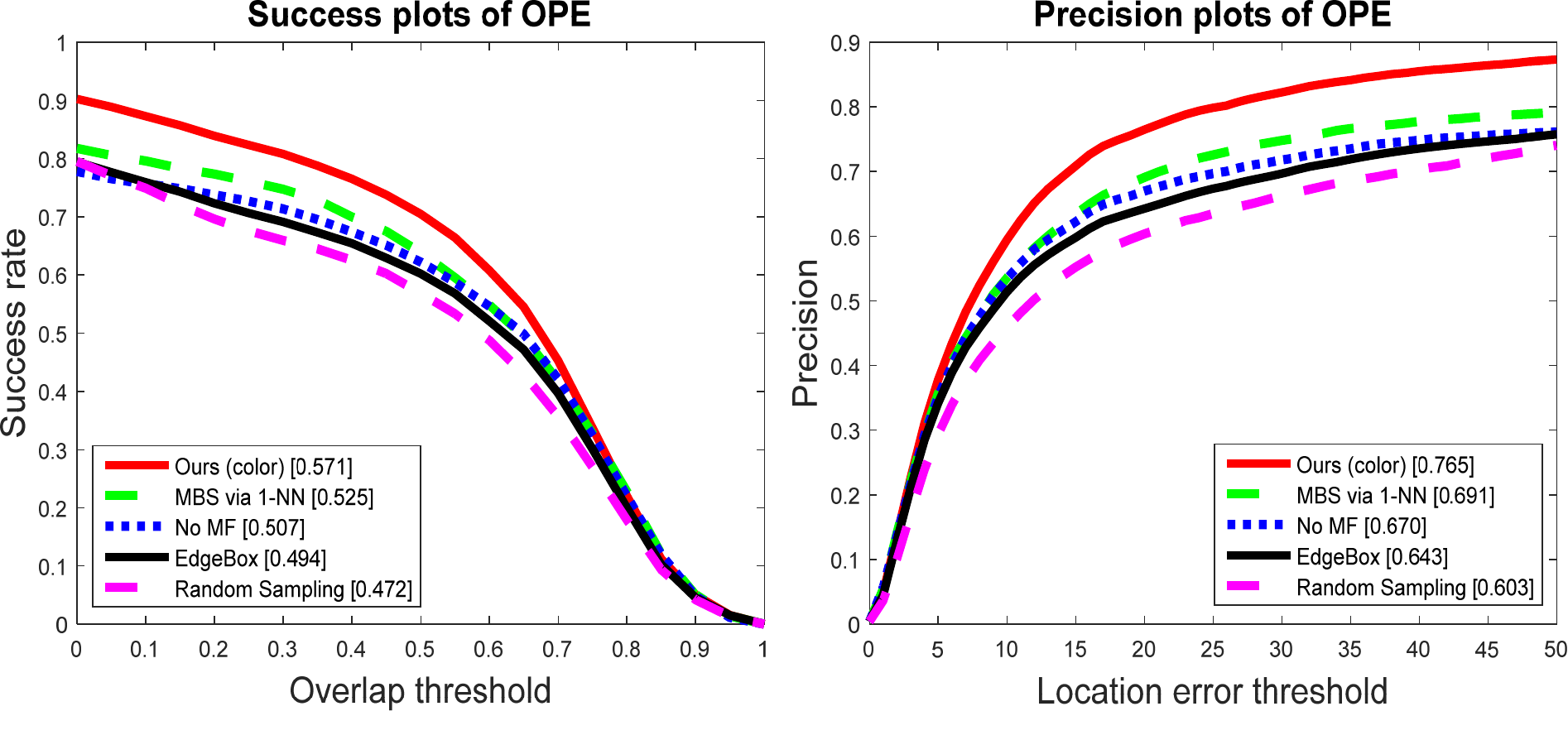}
\caption{\footnotesize Verification of four key components in our tracker on OTB-2013(36).
 ``MBS via $1$-NN" setting means that only the $1$-reciprocal nearest neighbor is utilized for region matching; ``No MF" setting denotes that the templates Tmpl$_r$ is also frequently updated as Tmpl$_e$ without the memory filtering strategy; ``Edgebox" setting means that the candidate regions are only generated by the EdgeBox approach; ``Random sampling" denotes that \texttt{Flow-r} process is retained and \texttt{Flow-e} process is removed.}
\label{remove}
\end{center}
\end{figure}
In this section, we firstly test the effects of several key components to see how they contribute to improving the final performance, and then investigate the parametric sensitivity of four parameters in the proposed tracker.

\subsubsection{Key component verification}
Several key components includes the scheme of multiple reciprocal nearest neighbors, memory filtering strategy, and the candidate generation scheme.
The influence of each component on the final tracking performance is illustrated in Fig.~\ref{remove}.

Firstly, to demonstrate that our multiple reciprocal nearest neighbors scheme is better than simply using one nearest neighbor, we compute MBS by only considering the single nearest neighbor (\emph{i.e.} ``MBS via $1$-NN'').
We see that ``1-NN" setting leads to the reduction of 4.6\% on average overlap rate when compared with the adopted ``MBS'' strategy.
Therefore, the utilization of multiple nearest neighbors in our tracker enhances the discriminative ability of the existing $1$-reciprocal nearest neighbor scheme.

Secondly, to investigate how the memory filtering strategy contributes to improving the final performance, we remove the memory filtering manipulation from our TM$^3$ tracker (\emph{i.e.} ``No MF'') and see the performance.
 The average success rate of such ``No MF" setting is as low as 50.7\%, with a 6.4\% reduction compared with the complete TM$^3$ tracker.
 As a result, the memory filtering strategy plays an importance role in obtaining satisfactory tracking performance.

 Lastly, to illustrate the effectiveness of two different candidate generation types with the multiple templates scheme, we design two experimental settings: ``Edgebox" and ``Random Sampling".
 In ``Random Sampling" setting, only \texttt{Flow-r} process is retained, which further causes to the invalidation of Tmpl$_e$ and the fusion scheme.
 In this case, ``Random Sampling" directly outputs $\mathcal{I}_r^*$ as the final tracking result without any fusion scheme, and then produces only one template Tmpl$_r$ for updating.
 As a consequence, the average overlap rate dramatically decreases from the original 57.1\% to 47.2\% if only \texttt{Flow-r} process is used.
 Likewise, in the ``Edgebox" setting, \texttt{Flow-e} process is retained and \texttt{Flow-r} process is removed, so only Tmpl$_r$ is involved for template updating.
Such setting achieves 49.4\% success rate and 64.3\% precision rate, which are much lower than the original result.

\subsubsection{Parametric sensitivity}
\label{sec:pse}

\label{sec:drep}
\begin{figure}
\begin{center}
\includegraphics[width=0.4\textwidth]{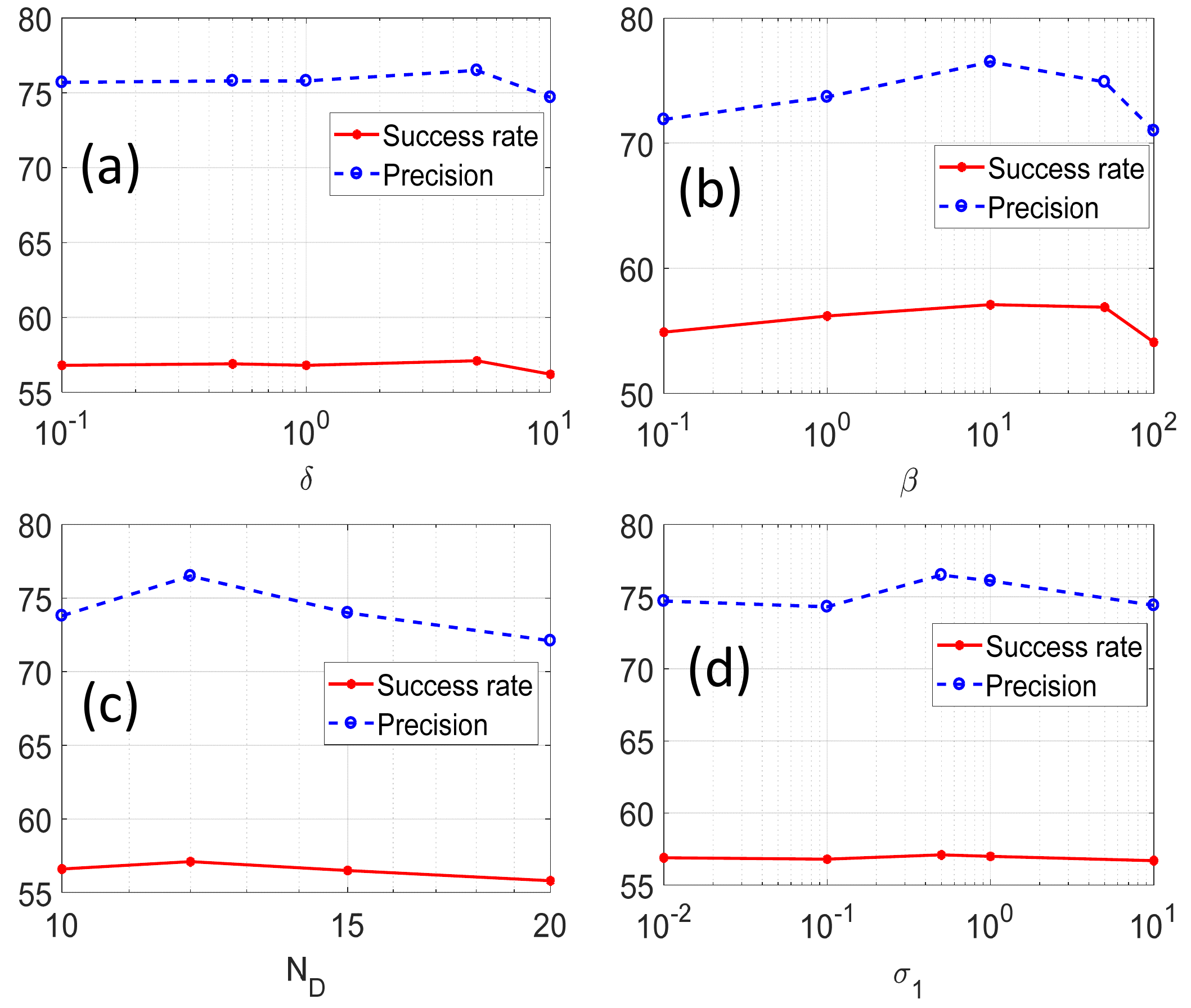}
\caption{\footnotesize Tracking performance of success rate and precision versus four varying parameters on OTB-2013(36).}
\label{param}
\end{center}
\end{figure}
Here we investigate the parametric sensitivity of the number of atoms in $\mathcal{D}$, the kernel width $\sigma_1$ in Eq.~\eqref{mBBPdef}, and two regularized parameters $\delta$ and $\beta$ in Eq.~\eqref{mainss}.
Fig.~\ref{param} illustrates that the proposed TM$^3$ tracker is robust to the variations of these parameters, so they can be easily tuned for practical use.

\section{Conclusion}
\label{sec:conclusion}
This paper proposes a novel template matching based tracker named TM$^3$ to address the limitations of existing CFTs.
The proposed MBS notably improves the discriminative ability of our tracker as revealed by both empirical and theoretical analyses.
Moreover, the memory filtering strategy is incorporated into the tracking framework to select ``representative" and ``reliable" previous tracking results to construct the current trustable templates, which greatly enhances the robustness of our tracker to appearance variations.
Experimental results on two benchmarks indicate that our TM$^3$ tracker equipped with the multiple reciprocal nearest neighbor scheme and the memory filtering strategy can achieve better performance than other state-of-the-art trackers.

\appendices
\section{The proof of Lemma \ref{proEsim}}
\label{sec:app1}
\changed{M1.4}{\link{R1.4}}{This section aims to simplify the formulation of $\mathbb{E}_{\text{MBS}}(\mathcal{P},\mathcal{Q})$ defined by Definition~\ref{definite} as follows\footnote{Our simplification process of $\mathbb{E}_{\text{MBS}}(\mathcal{P},\mathcal{Q})$ is similar to that of $\mathbb{E}_{\text{BBS}}(\mathcal{P},\mathcal{Q})$. Please refer to Section 3.1 in \cite{Dekel_BBS_mini}.}.}

Due to the independence among image patches, the integral in Eq.~(\ref{EBBP}) can be decoupled as:
\begin{equation} \label{eq:E_bbp2}
\begin{array}{l}\
\begin{split}
&\mathbb{E}_{\text{MBS}}(\mathcal{P},\mathcal{Q})=\\
&\int\limits_{p_1}\cdots\int\limits_{p_{N}} \int\limits_{q_1}\cdots\int\limits_{q_M} {\text{MBP}}(\mathbf{p}_i,\mathbf{q}_j) \prod\limits_{k=1}^{N}  f_{P}(p_k) \prod\limits_{l=1}^{M} f_{Q}(q_l)\mathrm{d}P\mathrm{d}Q,  \\
\end{split}
\end{array}
\end{equation}
where $\mathrm{d}P=\mathrm{d}p_1 \cdot \mathrm{d}p_2\cdots \mathrm{d}p_N$, and $\mathrm{d}Q=\mathrm{d}q_1 \cdot \mathrm{d}q_2 \cdots \mathrm{d}q_M$.
By introducing the indictor $\mathbb{I}$, it equals to 1 when $ \mathbf{q}_j = \text{NN}_r(\mathbf{p}_i,\mathcal{Q})  \wedge
 \mathbf{p}_i =\text{NN}_s(\mathbf{q}_j,\mathcal{P})$.
  The similarity $\text{MBP}(\mathbf{p}_i,\!\mathbf{q}_j)$ in Eq.~\eqref{mBBPdef} can be reformulated as:
\begin{small}
\begin{equation}\label{BBPdis}
\begin{split}
  &\text{MBP}(\mathbf{p}_i,\!\mathbf{q}_j)=\\
  &\mathrm{exp}\Bigg\{\!\! {-\frac{1}{\sigma_1}\!\!\sum_{k=1,k\neq i}^N \!\! \! \!\!\mathbb{I}\big[\!d(\mathbf{p}_k, \mathbf{q}_j)\! \leq\! d(\mathbf{p}_i, \mathbf{q}_j)\big]\!\! \cdot \!\!\!\!\!\sum_{l=1,l\neq i}^M \!\! \!\!\mathbb{I}\big[d(\mathbf{q}_l, \mathbf{p}_i)\! \leq \!d(\mathbf{p}_i, \mathbf{q}_j)\big]} \!\!\Bigg\},
  \end{split}
\end{equation}
\end{small}
which shares the similar formulation of  ${\text{BBP}}(\mathbf{p}_i,\mathbf{q}_j)$ in \cite{Dekel_BBS_mini} (see in Eq.~(7) on Page 5).
And next, by defining:
\begin{equation}\label{cpk}
 Cp_k = \int\limits_{-\infty}^{\infty}  \mathbb{I}[d(\mathbf{p}_k, \mathbf{q}_j) \leq d(\mathbf{p}_i, \mathbf{q}_j)]f_{P}(p_k)\mathrm{d}p_k\,,
\end{equation}
and assuming $d(\mathbf{p},\mathbf{q})=\sqrt{(\mathbf{p}-\mathbf{q})^2} = |\mathbf{p}-\mathbf{q}|$, we can rewrite Eq.~\eqref{cpk} as:
\begin{equation}\label{cpkn}
  Cp_k  \!\!= \!\!\!\!\int\limits_{-\infty}^{\infty}\!\!  \mathbb{I}\big[\mathbf{p}_k \!<\! \mathbf{q}_j^- \vee \mathbf{p}_k \!>\! \mathbf{q}_j^+ \big]f_P(p_k)\mathrm{d}p_k
  \!=\!F_{P}(q_j^+)-F_{P}(q_j^-)\,.
\end{equation}
Similarly, $Cq_l$ is:
\begin{equation}\label{cql}
Cq_l \!= \!\!\!\!\int\limits_{-\infty}^{\infty}\!\! \mathbb{I}\big[d(\mathbf{q}_l, \mathbf{p}_i) \leq d(\mathbf{p}_i, \mathbf{q}_j)\big]f_{Q}(q_l)\mathrm{d}q_l\!= \!F_{Q}(p_i^+)-F_{Q}(p_i^-)\,.
\end{equation}
Note that $Cp_k$ and $Cq_l$ only depend on $\mathbf{p}_i$, $\mathbf{q}_j$, and the underlying distributions $f_P(p)$ and $f_Q(q)$.
Therefore, $\mathbb{E}_{\text{MBS}}(\mathcal{P},\mathcal{Q})$ \changed{M1.5.31}{\link{R1.5}}{can be reformulated as}:
\begin{equation*}\label{rssim}
\begin{split}
&\mathbb{E}_{\text{MBS}}(\mathcal{P},\mathcal{Q})= \int_{p_i}\int_{q_j}\mathrm{d}p_i\mathrm{d}q_jf_{P}(p_i)f_{Q}(q_j)
\text{MBP}(\mathbf{p}_i,\mathbf{q}_j)\,.
\end{split}
\end{equation*}
Using the Taylor expansion with second-order approximation $\mathrm{exp}(-\frac{1}{\sigma}x)=1-\frac{1}{\sigma}x+\frac{1}{2\sigma^2}x^2$ and $NCp_k=\sum_{k=1,k\neq i}^N \mathbb{I}\big[d(\mathbf{p}_k, \mathbf{q}_j)\! \leq\! d(\mathbf{p}_i, \mathbf{q}_j)\big]$ by Eq.~\eqref{cpkn}, and $MCq_l = \sum_{l=1,l\neq i}^M \mathbb{I}\big[d(\mathbf{q}_l, \mathbf{p}_i)\! \leq \!d(\mathbf{p}_i, \mathbf{q}_j)\big]$ Eq.~\eqref{cql}, we have:
\begin{equation}\label{Tysim}
\begin{split}
&\mathbb{E}_{\text{MBS}}(\mathcal{P},\mathcal{Q})=\\
&1-\frac{1}{\sigma_1}\int_{p_i}\int_{q_j}(MCq_l)\cdot(NCp_k) f_{P}(p_i)f_{Q}(q_j)\mathrm{d}p_i\mathrm{d}q_j\\
&+\frac{1}{2\sigma_1^2}\int_{p_i}\int_{q_j}(MCq_l)^2\cdot(NCp_k)^2 f_{P}(p_i)f_{Q}(q_j)\mathrm{d}p_i\mathrm{d}q_j\,.
\end{split}
\end{equation}
Finally, after some straightforward algebraic manipulations, the $\mathbb{E}_{\text{MBS}}$ in Eq.~\eqref{Esim} can be easily obtained.

\section{The proof of Lemma \ref{proDsim}}
\label{sec:app2}
This section begins with the formulation of $\mathbb{E}_{\text{MBS}^2}(\mathcal{P},\mathcal{Q})$, and then computes the variance
 $ \mathbb{V}_{\text{MBS}}(\mathcal{P},\mathcal{Q}) =\mathbb{E}_{\text{MBS}^2}(\mathcal{P},\mathcal{Q}) - \mathbb{E}^2_{\text{MBS}}(\mathcal{P},\mathcal{Q})$ by Lemma \ref{proEsim}.

First, the similarity metric ${\text{MBP}^2}(\mathcal{P},\mathcal{Q})$ is obtained by Eq.~\eqref{BBPdis}, namely:
\begin{small}
\begin{equation}\label{sim2}
\begin{split}
  &\text{MBP}^2(\mathbf{p}_i,\mathbf{q}_j)=\\
  &\mathrm{exp}\Bigg\{\!\!\!-\!\frac{2}{\sigma_1}\!\!\sum_{k=1, k \neq i}^N \!\!\!\!\! \mathbb{I}\big[d(\mathbf{p}_k, \!\mathbf{q}_j)\! \leq\! d(\mathbf{p}_i, \! \mathbf{q}_j)\big]\!\! \cdot \!\!\!\!\sum_{l=1,l \neq i}^M\! \!\!\mathbb{I}\big[d(\mathbf{q}_l, \! \mathbf{p}_i) \!\leq\! d(\mathbf{p}_i,\! \mathbf{q}_j)\big]\!\!\Bigg\}.
  \end{split}
\end{equation}
\end{small}

Similar to the above derivation of $\mathbb{E}_{\text{MBS}}(\mathcal{P},\mathcal{Q})$ in Lemma~\ref{proEsim}, $\mathbb{E}_{\text{MBS}^2}(\mathcal{P},\mathcal{Q})$ can be computed as:
\begin{small}
\begin{equation}\label{Esim2}
\begin{split}
&\mathbb{E}_{\text{MBS}^2}(\mathcal{P},\mathcal{Q}) = 1-\\
&\frac{2MN}{\sigma_1}\!\!\! \iint\limits_{-\infty}^{~~~~\infty} \!\!\! \big[F_{P}(q^+\!)\!\!-\!\!F_{P}(q^-\!)\!\big]\!\!
\big[F_{Q}(p^+\!)\!\!-\!\!F_{Q}(p^-\!)\!\big]\!
f_{P}(p)f_{Q}(q)
\mathrm{d}p\mathrm{d}q+\\
&\frac{2M^2N^2}{\sigma_1^2}\!\!\!\iint\limits_{-\infty}^{~~~~\infty}\!\!\!
\big[\!F\!_{P}(q^+\!)\!\!-\!\!F\!_{P}(q^-\!)\!\big]\!^2\!
\big[\!F\!_{Q}(p^+\!)\!\!-\!\!F_{Q}(p^-\!)\!\big]\!^2\!
f_{P}(p)f_{Q}(q)\mathrm{d}p\mathrm{d}q.
\end{split}
\end{equation}
\end{small}

Next, $\mathbb{E}^2_{\text{MBS}}(\mathcal{P},\mathcal{Q})$ can be obtained by Lemma \ref{proEsim} with its second-order approximation, namely:
\begin{small}
\begin{equation}\label{sE2sim}
\begin{split}
&\mathbb{E}^2_{\text{MBS}}(\mathcal{P},\mathcal{Q}) = 1+\\
&\frac{M^2N^2}{\sigma_1^2}\! \Bigg\{\!\iint\limits_{-\infty}^{~~~~\infty} \!\! [\!F_{P}(q^+)\!\!-\!\!F_{P}(q^-)\!]
[F_{Q}(p^+)\!\!-\!\!F_{Q}(p^-)]f_{P}(p)f_{Q}(q)
\mathrm{d}p\mathrm{d}q\!\Bigg\}^{\!\!2}\\
&
-\frac{2MN}{\sigma_1}\!\!\! \iint\limits_{-\infty}^{~~~~\infty} \!\!\! \big[F_{P}(q^+\!)\!\!-\!\!F_{P}(q^-\!)\!\big]\!\!
\big[F_{Q}(p^+\!)\!\!-\!\!F_{Q}(p^-\!)\!\big]\!
f_{P}(p)f_{Q}(q)
\mathrm{d}p\mathrm{d}q\\
&
+\frac{M^2N^2}{\sigma_1^2}\!\!\!\iint\limits_{-\infty}^{~~~~\infty}\!\!\!
\big[\!F\!_{P}(q^+\!)\!\!-\!\!F\!_{P}(q^-\!)\!\big]^2\!
\big[\!F\!_{Q}(p^+\!)\!\!-\!\!F_{Q}(p^-\!)\!\big]^2\!\!
f_{P}(p)f_{Q}(q)\mathrm{d}p\mathrm{d}q.
\end{split}
\end{equation}
\end{small}

As a result, Lemma \ref{proDsim} can be easily proved after some straightforward algebraic manipulations on Eq.~\eqref{Esim2} and Eq.~\eqref{sE2sim}.
\section{The proof of Lemma \ref{Lemmare}}
\label{sec:app3}
By Lemma \ref{proEsim} and Eq.~\eqref{sE2sim}, $\mathbb{E}_{\text{MBS}^2}(\mathcal{P},\mathcal{Q})$ can be represented by:
\begin{equation*}\label{relationship}
\begin{split}
  &\mathbb{E}_{\text{MBS}^2}(\mathcal{P},\mathcal{Q}) = 4\mathbb{E}_{\text{MBS}}(\mathcal{P},\mathcal{Q})-3+\\
  &\frac{3MN}{\sigma_1} \!\! \iint\limits_{-\infty}^{~~~~\infty} \!\!\! \big[F_{P}(q^+\!)\!\!-\!\!F_{P}(q^-\!)\!\big]\!\!
\big[F_{Q}(p^+\!)\!\!-\!\!F_{Q}(p^-\!)\!\big]\!
f_{P}(p)f_{Q}(q)
\mathrm{d}p\mathrm{d}q\,.
\end{split}
\end{equation*}
Therefore, by Lemma~\ref{proDsim} and Eq.~\eqref{Tysim}, \changed{M1.2}{\link{R1.2}}{we have:}
\begin{small}
\begin{equation}\label{diff}
\begin{split}
 &\mathbb{E}_{\text{MBS}^2}(\mathcal{P},\mathcal{Q})-
 \mathbb{E}_{\text{MBS}}(\mathcal{P},\mathcal{Q}) =
 3\mathbb{E}_{\text{MBS}}(\mathcal{P},\mathcal{Q})-3\\
  &+\frac{3MN}{\sigma_1} \!\! \iint\limits_{-\infty}^{~~~~\infty} \!\!\! \big[F_{P}(q^+\!)\!\!-\!\!F_{P}(q^-\!)\!\big]\!\!
\big[F_{Q}(p^+\!)\!\!-\!\!F_{Q}(p^-\!)\!\big]\!
f_{P}(p)f_{Q}(q)
\mathrm{d}p\mathrm{d}q\\
 &
 =\frac{3M^2N^2}{2\sigma_1^2}\!\!\!\iint\limits_{-\infty}^{~~~~\infty}\!\!\!
\big[\!F\!_{P}(q^+\!)\!\!-\!\!F\!_{P}(q^-\!)\!\big]^{\!2}\!
\big[\!F\!_{Q}(p^+\!)\!\!-\!\!F_{Q}(p^-\!)\!\big]^{\!2}\!\!
f_{P}(p)f_{Q}(q)\mathrm{d}p\mathrm{d}q\\
 &
> 0\,,
\end{split}
\end{equation}
\end{small}
which completes the proof.

\section*{Acknowledgements}
The authors would like to thank Cheng Peng from Shanghai Jiao Tong University for his work on algorithm comparisons, and also sincerely appreciate the anonymous reviewers for their insightful comments.
\bibliographystyle{IEEEtran}



\end{document}